%% file: main.tex
\documentclass{article}
\usepackage[final,nonatbib]{neurips_2021}

\usepackage[utf8]{inputenc} 
\usepackage[T1]{fontenc}    
\usepackage{url}            
\usepackage{booktabs}       
\usepackage{amsfonts}       
\usepackage{nicefrac}       
\usepackage{microtype}      
\usepackage{lipsum}		
\usepackage{graphicx}
\usepackage{doi}
\usepackage{subcaption}
\usepackage{todonotes}
\usepackage{amsmath}
\usepackage{amsthm}
\usepackage[ruled]{algorithm2e}
\usepackage{algpseudocode}
\usepackage{mathtools}
\usepackage{xparse}
\usepackage[inline]{enumitem}
\usepackage{tablecoloring}
\usepackage[numbers]{natbib}
\usepackage{xr}
\usepackage{hyperref}       

\setlist[itemize]{leftmargin=*}
\setlist[enumerate]{leftmargin=*}
\let\cite\citep

\newcommand{\ours}{\textbf{\textsc{sqaler}}} 

\newcommand{\kg}{\mathcal{G}}
\newcommand{\nodes}{\mathcal{V}}
\newcommand{\relations}{\mathcal{R}}
\newcommand{\edges}{\mathcal{E}}
\newcommand{\query}{Q}
\newcommand{\qlen}{|\query|}

\newcommand{\lquery}{\mathsf{Q}}
\newcommand{\llit}{\mathsf{e}}
\newcommand{\lconj}{\mathsf{C}}
\newcommand{\lvar}{\mathsf{V}}
\newcommand{\lrel}{r}
\newcommand{\ltarget}{\lvar_{?}}

\newcommand{\beams}{\mathcal{B}}

\newcommand{\reach}{\textit{reach}_{\kg}}
\newcommand{\answers}{\mathcal{A}_{\query}}
\newcommand{\seed}{\mathcal{V}_{\query}}
\newcommand{\coal}[1]{\tilde{#1}}
\newcommand{\coalkg}{\coal{\kg}_{\query}}
\newcommand{\coalnodes}{\coal{\nodes}_{\query}}
\newcommand{\coalrel}{\coal{\relations}_{\query}}
\newcommand{\coaledges}{\coal{\edges}_{\query}}
\newcommand{\candidates}{\coal{\mathcal{A}}_{\query}}

\newcommand{\relseq}{R}
\newcommand{\relmodel}{\phi}

\newcommand{\relweights}{\relmodel}
\newcommand{\edgemodel}{\psi}

\newcommand{\subkg}{\kg(\candidates)}
\newcommand{\subnodes}{\nodes(\candidates)}

\newcommand{\maxdegout}{d_{\textit{max}}^{+}}
\newcommand{\tmax}{\tau_{\textit{max}}}

\newcommand{\dechidd}{\mathbf{X}^l_{t}}
\newcommand{\dechidditem}{\mathbf{x}^l_{t}}
\newcommand{\decself}{\mathbf{\bar{x}}^l_{t}}
\newcommand{\decenc}{\bar{\mathbf{x}}^{\query,l}_{t}}
\newcommand{\decgraph}{\bar{\mathbf{x}}^{\relseq,l}_{t}}

\newtheorem{proposition}{Proposition}
\newtheorem{lemma}{Lemma}
\mathtoolsset{showonlyrefs}

\begin{document}

\title{SQALER: Scaling Question Answering by Decoupling Multi-Hop and Logical Reasoning}


\author{Mattia Atzeni \\
	IBM Research, EPFL \\
	Switzerland \\
	\texttt{atz@zurich.ibm.com} \\
	\And
	Jasmina Bogojeska \\
	IBM Research \\
	Switzerland \\
	\texttt{jbo@zurich.ibm.com} \\
	\And
	Andreas Loukas \\
	EPFL \\
	Switzerland \\
	\texttt{andreas.loukas@epfl.ch} \\
}



\makeatletter
\let\latexparagraph\paragraph
\RenewDocumentCommand{\paragraph}{som}{%
  \IfBooleanTF{#1} 
    {\latexparagraph*{#3}}
    {\IfNoValueTF{#2}
       {\latexparagraph{\maybe@addperiod{#3}}}
       {\latexparagraph[#2]{\maybe@addperiod{#3}}}%
  }%
}
\newcommand{\maybe@addperiod}[1]{%
  #1\@addpunct{.}%
}
\makeatother


\maketitle

\begin{abstract}
State-of-the-art approaches to reasoning and question answering over knowledge graphs (KGs) usually scale with the number of edges and can only be applied effectively on small instance-dependent subgraphs. In this paper, we address this issue by showing that multi-hop and more complex logical reasoning can be accomplished separately without losing expressive power. Motivated by this insight, we propose an approach to multi-hop reasoning that scales linearly with the number of relation types in the graph, which is usually significantly smaller than the number of edges or nodes. This produces a set of candidate solutions that can be provably refined to recover the solution to the original problem. Our experiments on knowledge-based question answering show that our approach solves the multi-hop MetaQA dataset, achieves a new state-of-the-art on the more challenging WebQuestionsSP, is orders of magnitude more scalable than competitive approaches, and can achieve compositional generalization out of the training distribution.
\end{abstract}

\section{Introduction}
Reasoning, namely the ability to infer conclusions and draw predictions based on existing knowledge, is a hallmark of human intelligence.
Infusing the same ability into machine learning models has been a major challenge \citep{Marcus2020,LakeBaroni2018} and has historically
required complex systems made of several hand-crafted or learned components \citep{Yao2014,Ferrucci2010}.
Recently, the paradigm has shifted to deep learning approaches \cite{Sun2020,Sun2019}, where neural networks are used to reason over structured knowledge or a text corpus.
In this work, we assume that the source of knowledge is a structured knowledge graph (KG) and we tackle the problem of \textit{knowledge-based question answering} (KBQA), namely finding answers to natural language queries involving multi-hop and logical reasoning over the KG.

Answering queries over a knowledge graph involves many challenges, among which scalability is a major issue. Real-world KGs often contain millions of nodes and even a 2-hop neighborhood of the entities mentioned in the query may comprise tens of thousands of nodes.
Many state-of-the-art approaches \cite{Sun2018,Sun2019,Saxena2020} address the challenge of scalability by building small query-dependent subgraphs. 
To this end, they usually use simple heuristics \citep{Sun2018} or, in some cases, iterative procedures based on learned classifiers \citep{Sun2019}.
This preprocessing step is usually needed because each forward pass in end-to-end neural networks for KBQA scales at least linearly with the number of edges in the  subgraph.
Training neural networks involves repeated evaluation, which renders even a linear complexity impractical for graphs of more than a few tens thousands of nodes.  

In order to address this issue, we introduce a novel approach called \ours{} (Scaling Question Answering by Leveraging Edge Relations). 
%
The method first learns a model that generates a set of candidate answers (entities in the KG) by \textit{multi-hop reasoning}: 
the candidate solutions are obtained by starting from the set of entities mentioned in the question and seeking those that provide an answer by chained relational following operations.
We  refer to this module as the \textit{relation-level} model.
%
We show that this multi-hop reasoning step can be done efficiently and provably generates a set of candidates including all the actual answers to the original question. \ours{} then uses a second-stage \textit{edge-level} model that recovers the real answers by performing logical reasoning on a subgraph in the vicinity of the candidate solutions. 
A visual summary of our approach is depicted in Figure \ref{fig:overview}.

The main contributions and takeaway messages of this work are the following:
\begin{enumerate}
\item KBQA can be addressed by first performing multi-hop reasoning on the KG and then refining the result with more sophisticated logical reasoning without losing expressive power (we will elaborate this claim in more details in Section \ref{sec:analysis}).
\item Multi-hop reasoning can be accomplished efficiently with a method that scales linearly with the number of relation types in the KG, which are usually significantly fewer than the number of facts or entities.
\end{enumerate}


In the remainder of the paper, we first provide an extensive overview of our approach and a theoretical analysis of the expressive power and the computational complexity of \ours{}. Our experimental results show that \ours{} achieves better reasoning performance than state-of-the-art approaches, generalizes compositionally out of the training distribution, and scales to the size of real-world knowledge graphs with millions of entities.

\section{Scaling KBQA with relation and edge-level reasoning}
\label{sec:approach}

This section provides a detailed description of our approach.
We start by defining the problem formally and giving an intuitive overview of \ours{}.
Then, we discuss the approach in more details and we analyze its computational complexity and expressive power

\paragraph{Problem statement}
We denote a knowledge graph as $\kg = (\nodes, \relations, \edges)$, where $v \in \nodes$ represents an entity or node in $\kg$, $r \in \relations$ is a relation type, and we write $v \xrightarrow{r} v'$ to denote an edge in $\edges$ labeled with relation type $r \in \relations$ between two entities $v, v' \in \nodes$.
We extend the same notation to sets of nodes by writing $\nodes_i \xrightarrow{r} \nodes_j$ if $\nodes_j = \{v_j \in \nodes \mid v_i \xrightarrow{r} v_j, v_i \in \nodes_i\}$.
Given a knowledge graph $\kg = (\nodes, \relations, \edges)$  and a natural language question $\query$,
expressed as a sequence of tokens $\query = (q_1, q_2, \dots, q_{\qlen})$,
in \textit{knowledge-based question answering} the objective is to identify a set of nodes $\answers \subseteq \nodes$ representing the correct answers to $\query$.
Following previous work \cite{Sun2018,Sun2019,Sun2020}, we assume that the set of entities mentioned in the question $\seed \subseteq \nodes$ is given. These nodes are also called the \emph{anchor nodes} of the question and in practice are commonly obtained using an entity-linking module.

\paragraph{Overview}
KBQA can be cast as an entity seeking problem on $\kg$ by translating $Q$ into a set of nodes $\seed \subseteq \nodes$ (the starting points of the search) and seeking for nodes that provide an answer~\cite{Sun2018,Sun2019,Sun2020}.
Attempting to find $\answers$ directly on $\kg$ is prohibitive in practice, as even the most efficient graph-based neural networks generally scale at least linearly with the number of edges. 
Our approach mitigates this issue by breaking the problem in two subproblems.

(a) We first utilize a \textit{relation-level} model $\relmodel$ to obtain a set of \textit{candidate answers} $\candidates$, such that $\answers \subseteq \candidates$. 
We refer to $\relmodel$ as ``\textit{relation-level}'' because, as we will see, it operates on the \textit{coalesced graph}, a simplified representation of $\kg$, where edges of the same relation type are coalesced. The coalesced graph is constructed before training and incurs a one-time linear cost. By exploiting it during training, the relation-level model scales with the number of (distinct) relation types in the KG, which are usually significantly fewer than the number of edges or entities.


(b) The candidate answers are then refined using an \textit{edge-level} model $\psi$ applied on a subgraph
%
$\subkg$ of the original knowledge graph in the vicinity of $\candidates$.
We should note that the refining step is not always necessary. Indeed, we found that a relation-level model is sufficient to perfectly solve tasks like multi-hop question answering \cite{Zhang2018}. 
Figure \ref{fig:overview} shows an overview of our approach.

\subsection{Relational coalescing for efficient knowledge seeking}
\label{sec:seeking}

Our approach relies on a relation-level model $\phi$ that operates as a knowledge seeker in $\kg$.
The model identifies a node $v$ as a candidate $v \in \candidates$ based on the sequence of relations that connect it with $\seed$.
This can be achieved by using a neural network $\relmodel$ to predict how likely it is that the correct answer is reached from $\seed$ by following a sequence of relations $R$.
%
%

%

\begin{figure}[t!]
\centering
\includegraphics[width=0.9\linewidth]{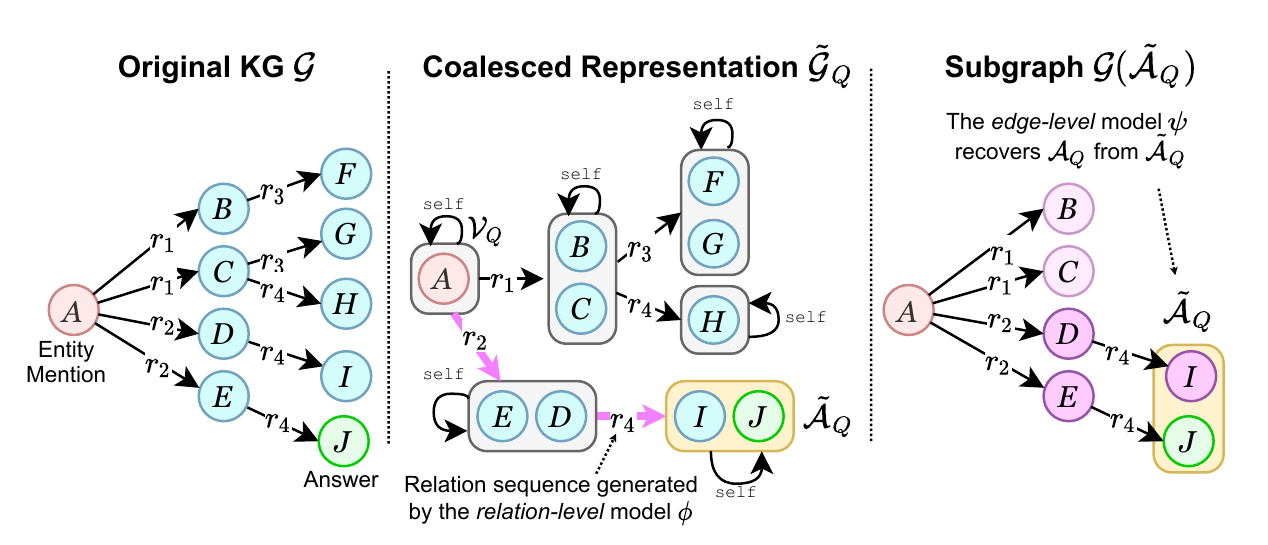}
\caption{Overview of our approach. A \textit{relation-level} model operates on a coalesced representation of the original KG to generate a set of candidate answers $\candidates$. This approximate solution is then refined by an \textit{edge-level} model applied on a subgraph of the original KG.}
\label{fig:overview}
\end{figure}

%
%

\paragraph{Reachability}
To define how our method works, it will help to formalize the concept of reachability.
Let $\relseq = (r_1, \dots, r_{|\relseq|})$ be a sequence of relations. We say that ``\textit{$v$ is $R$-reachable from $\seed$}'' if there exists a path $P = (v_1, \ldots, v_{|R|}, v)$ in $\kg$ such that: %
$$
v_1 \in \seed \quad  \text{and} \quad  v_{i} \xrightarrow{r_i} v_{i+1} \quad  \text{for every} \quad i = 1, \ldots, |R|.
$$
%
%
That is, we can reach $v$ by starting from a node in $\seed$ and following a sequence of edges with relation types $R$. 
We also denote by $\reach(\seed, \relseq)$ the set of nodes that are $\relseq$-reachable from $\seed$:
$$
    \reach(\seed, R) = \{v \in \nodes \mid v \ \text{is} \ R\text{-reachable from } \ \seed \}.
$$

\paragraph{Relational coalescing} 
Given a knowledge graph $\kg$, a question $\query$, and a set of entity mentions $\seed$, we consider a representation of the graph $\coalkg = (\coalnodes, \coalrel, \coaledges)$, which allows us to efficiently compute sets of nodes that are reachable from $\seed$. 
We refer to this representation as the question-dependent coalesced KG, because edges with the same relation type are coalesced, as shown in Figure \ref{fig:overview}. 
The nodes of $\coalkg$ are sets of nodes of $\kg$ that are reachable from $\seed$ by following any possible sequence of relations originating from $\seed$. The graph $\coalkg$ has an edge $\nodes_i \xrightarrow{r} \nodes_j$ if $\nodes_j$ is the set of nodes that are reachable from $\nodes_i$ by following relation $r$. For convenience, we include a relation type $\mathtt{self} \in \coalrel$ to denote self loops.
We refer the reader to Appendix \ref{app:coalesced_representation} for a formal definition of $\coalkg$.
The coalesced graph can be precomputed once as a preprocessing step for each question $\query$ and incurs a one-time linear cost.
In practice, however, we do not need to compute and store all the nodes in $\coalkg$ but only edge labels. This makes learning efficient because each forward/backward pass scales with the number of relation types and does not depend on the number of nodes or edges in the KG.

\paragraph{Knowledge seeking in $\coalkg$} The coalesced graph allows us to provide approximate answers to input questions in an efficient manner. Specifically, we seek $k \geq 1$ sequences of relations $\relseq^{\star}_i$, such that:
$$
    \answers \subseteq \tilde{\answers} = \bigcup_{i = 1}^{k} \reach(\seed, \relseq^{\star}_i).   
$$
%
%
We can achieve this by using a model $\relmodel$ that only considers relation sequences originating from $\seed$.
The model predicts the likelihood $\relweights: \coaledges \rightarrow [0, 1]$ of 
following a certain edge in a relation sequence from $\seed$ to $\candidates$. Then, given $\relseq = (r_1, \dots, r_{|\relseq|})$ and a node in the coalesced graph $\seed$, we can compute the likelihood of $\relseq$ by multiplying the likelihood of all edges traversed by $\relseq$ in $\coalkg$:
%
\[
\mathsf{P}(\relseq \mid \query, \coalkg, \seed) \propto \prod_{i=1}^{|\relseq|} \relweights(\reach(\seed, \relseq_{1\to i-1}), r_i, \reach(\seed, \relseq_{1\to i}) \mid \query),
\]
where $R_{1\to i} = (r_1, \dots, r_i)$ is the subsequence of $\relseq$ up to the $i$-th relation.
We generate $\candidates$ by selecting the top $k$ relation sequences $\relseq^{\star}_i$ with maximum likelihood $\mathsf{P}(\relseq^{\star}_i \mid \query, \coalkg, \seed)$.
This can be done by an efficient search algorithm, such as beam search starting from $\seed$.
Then, we compute $\candidates$ as the union of all target nodes of the selected relation sequences. More details about the knowledge-seeking algorithm are provided in Appendix $\ref{app:computational_complexity}$. 

\subsection{Refining the solution on the original KG}
\label{sec:refining}

In certain cases, like multi-hop question answering \cite{Zhang2018}, the set of candidate answers $\candidates$ may already be a reasonable estimate of $\answers$. We will substantiate this claim experimentally in Section \ref{sec:experiments}.
%
In general, however, we recover $\answers$ by using an \emph{edge-level} model $\edgemodel$ applied on a subgraph $\subkg$ of $\kg$.
Specifically, we construct $\subkg$ as the subgraph induced by the set of nodes $\subnodes$, which includes all nodes visited when following the top-$k$ relation sequences along with their neighbors (see Figure \ref{fig:overview} for an example).
%
Any existing method for KBQA can be used to instantiate $\edgemodel$ by running it on $\subkg$ rather than $\kg$. We opted to use a Graph Convolutional Network (GCN) conditioned on the input question with the same architecture as in \cite{Sun2018}. The edge-level model is constrained to predict an answer among the candidates generated by the relation-level model.

\subsection{Analysis of scalability and expressive power}
\label{sec:analysis}

This section provides a scalability analysis of our approach and shows that the relation-level model scales linearly with the number of relation types in the graph. Then, we analyse the expressive power of \ours{} and we show the class of supported logical queries.

\paragraph{Computational complexity}
As mentioned, we do not evaluate the likelihood $\relweights$ for all edges in $\coalkg$, but we generate the most likely relation sequences using a knowledge-seeking procedure based on the beam search algorithm. At any given time step, only the $\beta$ most likely relation sequences are retained and further explored at the next iteration.
Hence, the time complexity required by our algorithm is $\mathcal{O}(\tmax \cdot \beta \cdot \maxdegout(\coalkg))$, where $\tmax$ is the maximum allowed number of decoding time steps and $\maxdegout(\coalkg)$ is the maximum outdegree of $\coalkg$.
Note that $\maxdegout(\coalkg)$ is bounded by the number of relations in the graph, whereas $\tmax$ and $\beta$ are constant parameters of the algorithm and are usually small.
This gives a time complexity of:
\[
\mathcal{O}(\tmax \cdot \beta \cdot |\relations|) = \mathcal{O}(|\relations|).
\]
Hence, the knowledge-seeking algorithm scales linearly with the number of relations in the KG.
The space complexity is also $\mathcal{O}(\tmax \cdot \beta \cdot |\relations|)$.
A more detailed analysis is provided in Appendix \ref{app:computational_complexity}.

\paragraph{Expressive power}
Given a natural language question $\query$, we can represent the inferential chain needed to obtain $\answers$ from $\seed$ as a logical query $\lquery$ on $\kg$. As an example, the question in Figure \ref{fig:model_architecture}, \emph{``Who starred in films directed by George Lucas?''}, can be represented by the logical query:
$
\lquery[\ltarget] = \ltarget . \exists \mathsf{\lvar} : \mathsf{Directed}(\mathsf{George\_Lucas}, \lvar) \land \mathsf{Starred}(\lvar, \ltarget).
$
We denote with $\ltarget$ the target variable of the query and we say that $v \in \nodes$ satisfies $\lquery$ if $\lquery[v] = \mathsf{True}$. A query $\lquery$ is an \textit{existential positive first-order (EPFO) query} if it involves the existential quantification ($\exists$), conjunction ($\land$), and disjunction ($\lor$) \citep{Dalvi2012} of literals corresponding to relations in the KG. Each literal is of the form $\lrel(\lvar, \lvar')$, where
$\lvar$ is either a node in $\seed$ or an existentially quantified bound variable, and $\lvar'$ is either an existentially quantified bound variable or the target variable.
A literal $\lrel(\lvar, \lvar')$ is satisfied if $\lvar \xrightarrow{r} \lvar'$, for $r \in \relations$.
Any EPFO query can be represented in \textit{disjunctive normal form} (DNF) \cite{davey2002}, namely as a disjunction of conjunctions.
Note that, we do not consider queries with universal quantification ($\forall$), as we assume that in real-world KGs no entity connects to all the others.
Then, the following proposition holds for any knowledge graph and EPFO query.

\begin{proposition}
\label{prop:expressive_power}
Let $\kg = (\nodes, \relations, \edges)$ be a knowledge graph and $\seed \subseteq \nodes$ denote a set of entities in $\kg$. Let $\lquery$ be a valid existential positive first-order query on $\kg$ and let $n_{\lor}$ be the number of disjunction operators in the disjunctive normal form of $\lquery$. 
Then, there exist $k \leq n_{\lor} + 1$ sequences of relations $\relseq^{\star}_i \in \relations^*$ such that:
\[
\answers \subseteq \bigcup_{i = 1}^{k} \reach(\seed, \relseq^{\star}_i),
\]
where $\answers = \{v \in \nodes \mid \lquery[v] = \mathsf{True}\}$ is the denotation set of $\lquery$, namely the entities satisfying $\lquery$.
\end{proposition}

This shows that sampling $n_\lor + 1$ sequences of relations  allows generating a set of candidate answers $\candidates$ that does not miss any of the real answers $\answers$.
Then, assuming that the edge-level model $\edgemodel$ can recover $\answers$ from $\candidates$, our approach can be used to answer any EPFO query on $\kg$.
More details about the expressive power of \ours{} and the proof of Proposition \ref{prop:expressive_power} are provided in Appendix \ref{app:expressive_power}. 



\section{Architecture of the relation-level model}
\label{sec:model}
For the relation-level model $\relmodel$, we propose an auto-encoder, where the the decoder is constrained to follow sequences of relations in the coalesced representation $\coalkg$.
We train the network with weak supervision, assuming that a sequence of relations is correct if it reaches a set of candidate answers $\candidates$ that is the smallest reachable superset of the $\answers$. 
We found it useful to \textit{pretrain} the model in order to infuse knowledge from the KG. In this case, we train the model to predict a path in the KG, given the representations of the source and target nodes. More details about training strategies are given in Appendix \ref{app:training}.

The architecture of the model (see Figure \ref{fig:model_architecture}) includes three main components: a \textit{question encoder}, a \textit{relation encoder} and a \textit{graph-guided decoder}. We explain each one below.

\begin{figure*}[!t]
\centering
\includegraphics[width=0.9\linewidth]{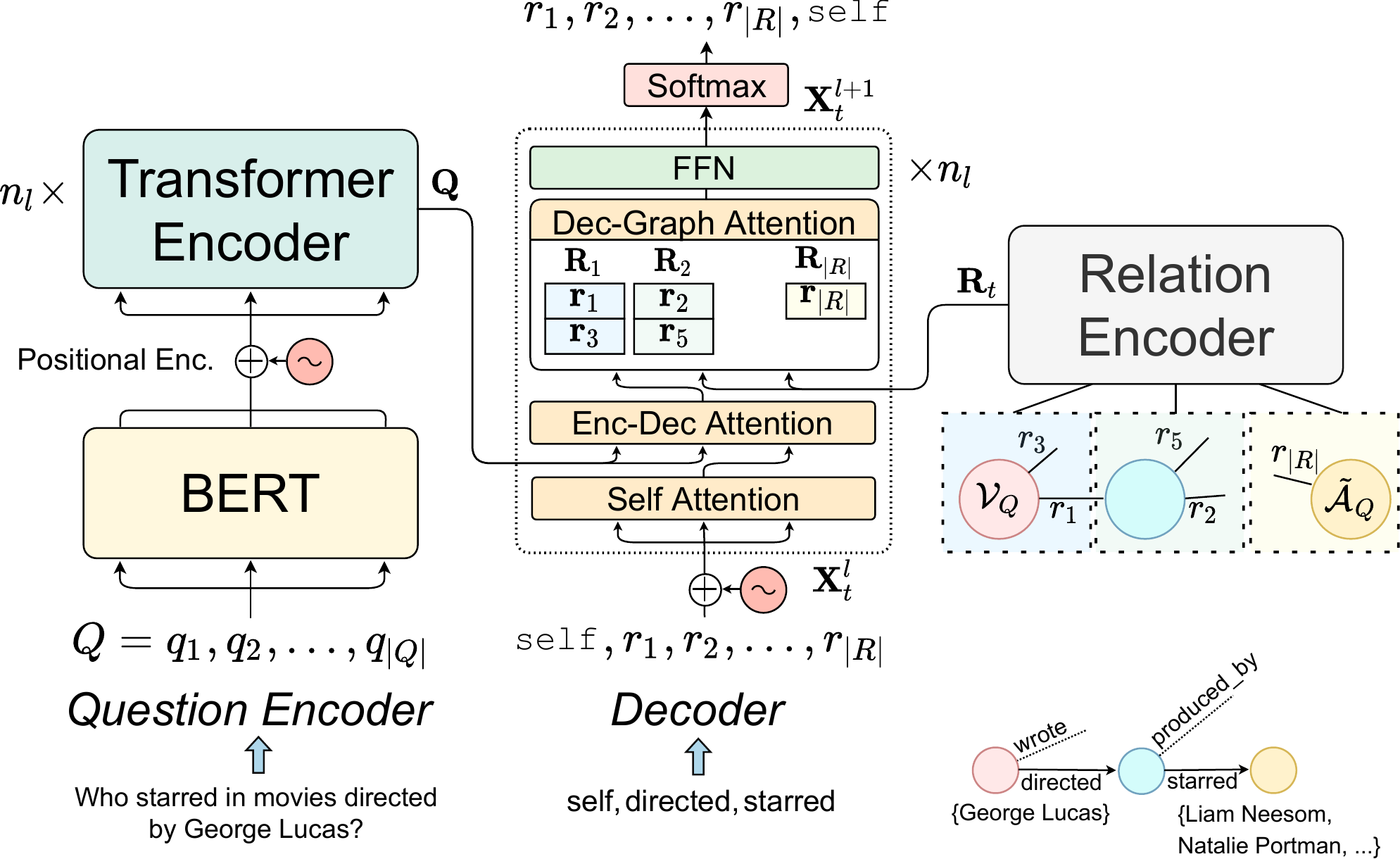}
\caption{Architecture of the \ours{} relation-level model. A question encoder is used to obtain a representation of tokens in the input natural language question. Then a graph-guided decoder is applied to obtain the likelihood of output relation sequences. The decoder is constrained to only attend to valid relations according to the structure of the coalesced KG.
}
\label{fig:model_architecture}
\end{figure*}

\paragraph{Question encoder}
The encoder receives as input a natural language question, which comprises a sequence of tokens $\query = (q_1, q_2, \dots, q_{|\query|})$.
The question is encoded using a pre-trained BERT \cite{devlin2019bert} model and processed with the same positional encoding technique used in \citep{Vaswani2017}.
The resulting embeddings are then fed into $n_{l} = 3$ transformer encoder layers \cite{Vaswani2017}.
This results in a matrix $\mathbf{Q} \in \mathbb{R}^{|\query| + 1 \times d_{\textit{model}}}$, where the first row vector is an overall representation of the whole query $Q$ (derived from the embedding of the \texttt{[CLS]} token introduced by BERT) and each remaining row represents the final $d_{\textit{model}}$-dimensional encoding of a token in the input question.

\paragraph{Relation encoder}
The relation encoder produces a representation $\mathbf{r} \in \mathbb{R}^{d_{\textit{model}}}$ for each relation type $r \in \relations$.
We decided to encode relations based on their surface form, with the same pre-trained BERT model used in the question encoder. In this case, only the embedding of the \texttt{[CLS]} token is used in order to get the final representation $\mathbf{r}$ of each relation type $r \in \relations$.
At inference time, or in case the BERT model is not fine-tuned, the embeddings of the relations can be precomputed as a preprocessing step to improve the efficiency of the approach.

\paragraph{Graph-guided decoder}
The decoder's job is to predict a sequence of relations leading from $\seed$ to $\candidates$ in $\coalkg$. At any time step $t$, it receives as input a sequence of relations $R_t = (\texttt{self}, r_1, \dots, r_{t - 1})$ and predicts the next relation $r_t$ ($\texttt{self}$ is used as a special token to denote the start of decoding). Note that the input sequence uniquely determines a node $\nodes_t$ in the graph $\coalkg$, namely the node reachable from $\seed$ by following $R_t$.
The decoder thus selects $r_t$ by choosing amongst the outgoing edges $\coal{\edges}_t$ of $\nodes_t$.
%
%
%
We use the same number of layers $n_l$ both for the question encoder and the decoder.
Let $\dechidd = [\mathbf{x}^l_0, \dots, \mathbf{x}^l_{t-1}]^{\top} \in \mathbb{R}^{t \times d_{\textit{model}}}$ denote the hidden state of the $l$-th layer of the decoder preceding time step $t$.
Note that $\mathbf{X}_t^0$ is the representation of the sequence $R_t$, obtained by using the relation encoder described above and the same positional encoding technique used in the question encoder.
For each decoder layer, we perform self-attention over the target sequence $\dechidd$ by computing:
\[
\decself = \textit{Attention}(\dechidditem, \dechidd, \dechidd),
\]
where $\textit{Attention}$ is a function that performs multi-head scaled dot-product attention \cite{Vaswani2017} with skip connections and layer normalization \cite{Ba2016}.
The above step allows each relation in the decoded sequence to attend to all the others predicted up to time step $t$.
We then let the result attend to the question as:
\[
\decenc = \textit{Attention}(\decself, \mathbf{Q}, \mathbf{Q}).
\]
This is done in order to update the current state of the decoder based on the input question.
Next, let $\mathbf{R}_t \in \mathbb{R}^{|\coal{\edges}_t| \times d_{\textit{model}}}$ denote the encoding of the relations labeling all edges in $\coal{\edges}_t$.
We constrain the decoded sequence to follow the structure of the graph by attending only to valid relations as follows:
\[
\decgraph = \textit{Attention}(\decenc, \mathbf{R}_t, \mathbf{R}_t).
\]
We get the hidden state of the next layer $\mathbf{x}_t^{l+1}$ by processing the result with a feed forward network.
The model outputs a categorical distribution $\relweights(e \mid \query) \in [0, 1]$ over the edges $e \in \coal{\edges}_t$, by applying a softmax function as follows:
\[
\relweights(\nodes_i \xrightarrow{r} \nodes_j \mid \query) = \frac{\exp(\mathbf{r}^\top \mathbf{x}_t^{n_l})}
{\sum_{\nodes_i' \xrightarrow{r'} \nodes_j' \in \coal{\edges}_t} \exp(\mathbf{r'}^\top \mathbf{x}_t^{n_l})},
\]
where $\mathbf{x}_t^{n_l}$ is the output of the final layer of the decoder, whereas $\mathbf{r}$ and $\mathbf{r}'$ denote the representations of relations $r$ and $r'$ respectively.

\section{Experiments}
\label{sec:experiments}
This section presents an evaluation of our approach with respect to both reasoning performance and scalability. 
We first show that \ours{} reaches state-of-the-art results on popular KBQA benchmarks and can generalize compositionally out of the training distribution. Then, we demonstrate the scalability of our approach on KGs with millions of nodes. We refer the reader to Appendix \ref{app:experimental_details} for more details about the experiments.

\subsection{Experimental setup}
\label{sec:setup}


\paragraph{Datasets}
We evaluate the reasoning performance of our approach on \textit{MetaQA} \citep{Zhang2018} and \textit{WebQuestionsSP} \citep{yih2015}.
\textit{MetaQA} includes multi-hop questions over the WikiMovies KB \citep{Miller2016} and we consider both 2-hop (\textbf{MetaQA 2}) and 3-hop (\textbf{MetaQA 3}) queries.
\textit{WebQuestionsSP} (\textbf{WebQSP}) comprises more complex questions answerable over a subset of Freebase \citep{freebase,Bollacker2008}, a large KG with millions of entities.
We further assess the compositional generalization ability of \ours{} on the \emph{Compositional Freebase Questions} (\textit{CFQ}) dataset \citep{Keysers2020}. Each question in \textit{CFQ} is obtained by composing primitive elements (\textit{atoms}). Whereas the training and test distribution of atoms are similar, the test set contains different \textit{compounds}, namely new ways of composing these atoms. \textit{CFQ} comprises three dataset splits (\textbf{MCD1}, \textbf{MCD2}, and \textbf{MCD3}), with maximal compound divergence (MCD) between the training and test distributions.
%
We refer the reader to Appendix \ref{app:datasets} for an extensive description of the datasets.

\paragraph{Evaluation protocol}
In our experiments on \textit{MetaQA} and \textit{WebQuestionsSP}, we assess the performance of three variants of our approach: (a) a version that only makes use of the relation-level model without the refinement step (\textbf{\ours{} -- Unrefined}), (b) a model that utilizes a key-value memory network to identify the correct answers from the candidates (\textbf{\ours{} -- KV-MemNN}), and (c) a model that uses a GNN architecture for the refinement step (\textbf{\ours{} -- GNN}), as explained in Section \ref{sec:refining}.
Following previous work \cite{Sun2018,Sun2019,Sun2020,Saxena2020}, we evaluate the models based on the \textit{Hits{@}1} metric.
On the \textit{CFQ} dataset, we evaluate the accuracy of the refined model with the GNN based on whether it predicts exactly the same answers given by the corresponding SPARQL query.

\subsection{Main results}

\paragraph{KBQA Performance}
Table \ref{tab:results} summarizes the results of our experiments on the two benchmark datasets.
For the two multi-hop \textit{MetaQA} datasets, we achieve state-of-the-art performance by only using the relation-level model of \ours{}.  
As shown in Table \ref{tab:results}, \ours{} outperforms all the baselines on \textbf{MetaQA 3}, demonstrating the ability of our approach to perform multi-hop reasoning over a KG.
%
For the more complex questions in the \textit{WebQuestionsSP} dataset, the unrefined \ours{} model achieves better performance than all but one (\textbf{EmQL}) of the baselines.
To achieve such performance, however, EmQL creates a custom set of logical operations tailored towards the specifics of the target KG and the kind of questions in the dataset, while our approach is agnostic with respect to such details. 
Combining the relation and edge-level models improves the performance on \textbf{WebQSP}. In particular, \textbf{\ours{} -- GNN} outperforms all considered baselines on the three datasets. 

\begin{table}[!htb]
\caption{Hits@1 on \textit{MetaQA} and \textit{WebQuestionsSP}}
\centering
\label{tab:results}
\resizebox{0.65\textwidth}{!}{%
\begin{tabular}{@{}lABC@{}}
\toprule
 & \textbf{MetaQA 2} & \textbf{MetaQA 3} & \textbf{WebQSP} 
\EndTableHeader \\
\midrule
\textbf{KV-MemNN} \citep{Miller2016} & 82.7 & 48.9 & 46.7 \\
\textbf{GRAFT-Net} \citep{Sun2018} & 94.8 & 77.7 & 70.3 \\
\textbf{ReifKB + mask} \citep{Cohen2020} & 95.4 & 79.7 & 52.7 \\
\textbf{PullNet} \citep{Sun2019} & 99.9 & 91.4 & 69.7 \\
\textbf{EmbedKGQA} \citep{Saxena2020} & 98.8 & 94.8 & 66.6 \\
\textbf{EmQL} \citep{Sun2020} & 98.6 & 99.1 & 75.5 \\
\midrule
\textbf{\ours{} -- Unrefined} & 99.9 & 99.9 & 70.6 \\
\textbf{\ours{} -- KV-MemNN} & 99.9 & 99.9 & 72.1 \\
\textbf{\ours{} -- GNN} & 99.9 & 99.9 & 76.1 \\ \bottomrule
\end{tabular}%
}
\end{table}

\paragraph{Compositional generalization}
In order to evaluate the compositional generalization ability of \ours{}, we performed additional experiments on the \emph{CFQ} dataset.
Table \ref{tab:cfq} shows the accuracy on the three MCD splits and the mean accuracy (\textbf{MCD-mean}) in comparison to the other methods in the leaderboard. Note that the other approaches address a semantic parsing task and require additional supervision, as they are trained to predict the target query. On the other hand, we aim to predict directly the set of answers to the input question.  
The experiment shows that \ours{} is able to achieve compositional generalization with an accuracy comparable to the state-of-the-art model on \textit{CFQ} for semantic parsing.

\begin{table}[!htb]
\centering
\caption{Accuracy and $95\%$ confidence interval on the \textit{CFQ} dataset}
\label{tab:cfq}
\resizebox{\textwidth}{!}{
\begin{tabular}{@{}lDDDD@{}}
\toprule
 & \textbf{MCD1} & \textbf{MCD2} & \textbf{MCD3} & \textbf{MCD-mean}
\EndTableHeader \\
\midrule
\textbf{LSTM + Attention} \citep{Keysers2020,hochreiter1997long,Bahdanau2015} & 0.289 $\pm$ 0.018 & 0.050 $\pm$ 0.008 & 0.108 $\pm$ 0.006 & 0.149 $\pm$ 0.011 \\
\textbf{Transformer} \citep{Keysers2020,Vaswani2017} & 0.349 $\pm$ 0.011 & 0.082 $\pm$ 0.003 & 0.106 $\pm$ 0.011 & 0.179 $\pm$ 0.009 \\
\textbf{Universal Transformer} \citep{Keysers2020,Dehghani2019} & 0.374 $\pm$ 0.022 & 0.081 $\pm$ 0.016 & 0.113 $\pm$ 0.003 & 0.189 $\pm$ 0.014 \\
\textbf{Evolved Transformer} \citep{Furrer2020,So2019} & 0.424 $\pm$ 0.010 & 0.093 $\pm$ 0.008 & 0.108 $\pm$ 0.002 & 0.208 $\pm$ 0.007 \\
\textbf{T5-11B} \citep{Raffel2020,Furrer2020} & 0.614 $\pm$ 0.048 & 0.301 $\pm$ 0.022 & 0.312 $\pm$ 0.057 & 0.409 $\pm$ 0.043 \\
\textbf{T5-11B-mod} \citep{Furrer2020,Guo2019} & 0.616 $\pm$ 0.124 & 0.313 $\pm$ 0.128 & 0.333 $\pm$ 0.023 & 0.421 $\pm$ 0.091 \\
\textbf{HPD} \citep{Guo2020} & 0.720 $\pm$ 0.075 & 0.661 $\pm$ 0.064 & 0.639 $\pm$ 0.057 & 0.673 $\pm$ 0.041 \\
\textbf{\ours{} -- GNN} & 0.734 $\pm$ 0.039 & 0.653 $\pm$ 0.040 & 0.627 $\pm$ 0.045 & 0.671 $\pm$ 0.041 \\ \bottomrule
\end{tabular}
}
\end{table}

\paragraph{Subgraph extraction}

We analyzed the candidate solutions produced by the relation-level model in
order to evaluate the suitability of our approach to building small question subgraphs that are likely to
contain the answers to a natural language question.
For this purpose, we computed the precision and recall of the
set of candidate answers with varying number of relation sequences sampled by the relation-level
model. Figure \ref{fig:attention_plots} shows the top relation sequences predicted by the relation-level model on two questions from the test set of \textit{WebQuestionsSP}. The precision and recall curves are shown in Figure \ref{fig:paths_prec_recall}. As expected, on \textit{MetaQA} the recall
is high for all values of $k$, because selecting the most likely sequence of relations is sufficient to
solve the multi-hop question answering task. On \textit{WebQuestionsSP}, only 3 sequences of relations are
sufficient to obtain a recall of $0.91$, and we can improve it to $0.95$ by generating still small subgraphs
consisting of only 10 sequences of relations.

\begin{figure*}[!t]
    \centering
    \includegraphics[width=\linewidth]{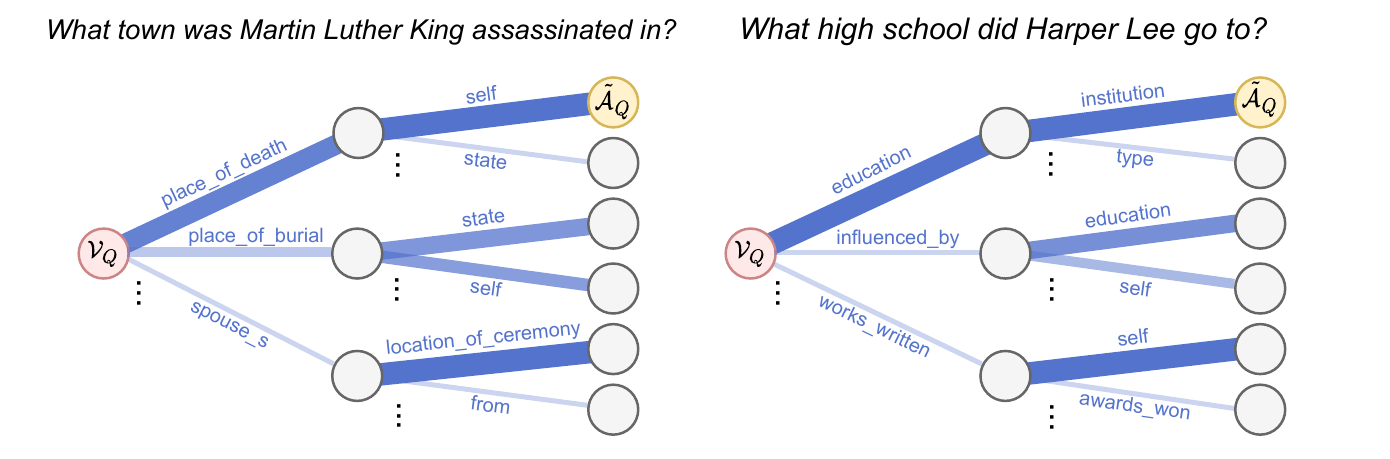}
    \caption{Attention weights given by the relation-level model to the edges of the coalesced graph for two questions in \textit{WebQuestionsSP}. Thicker and darker edges represent higher attention weights.}
    \label{fig:attention_plots}
\end{figure*}

\begin{figure*}[t]
\begin{subfigure}{0.33\linewidth}
\centering
\includegraphics[width=\linewidth]{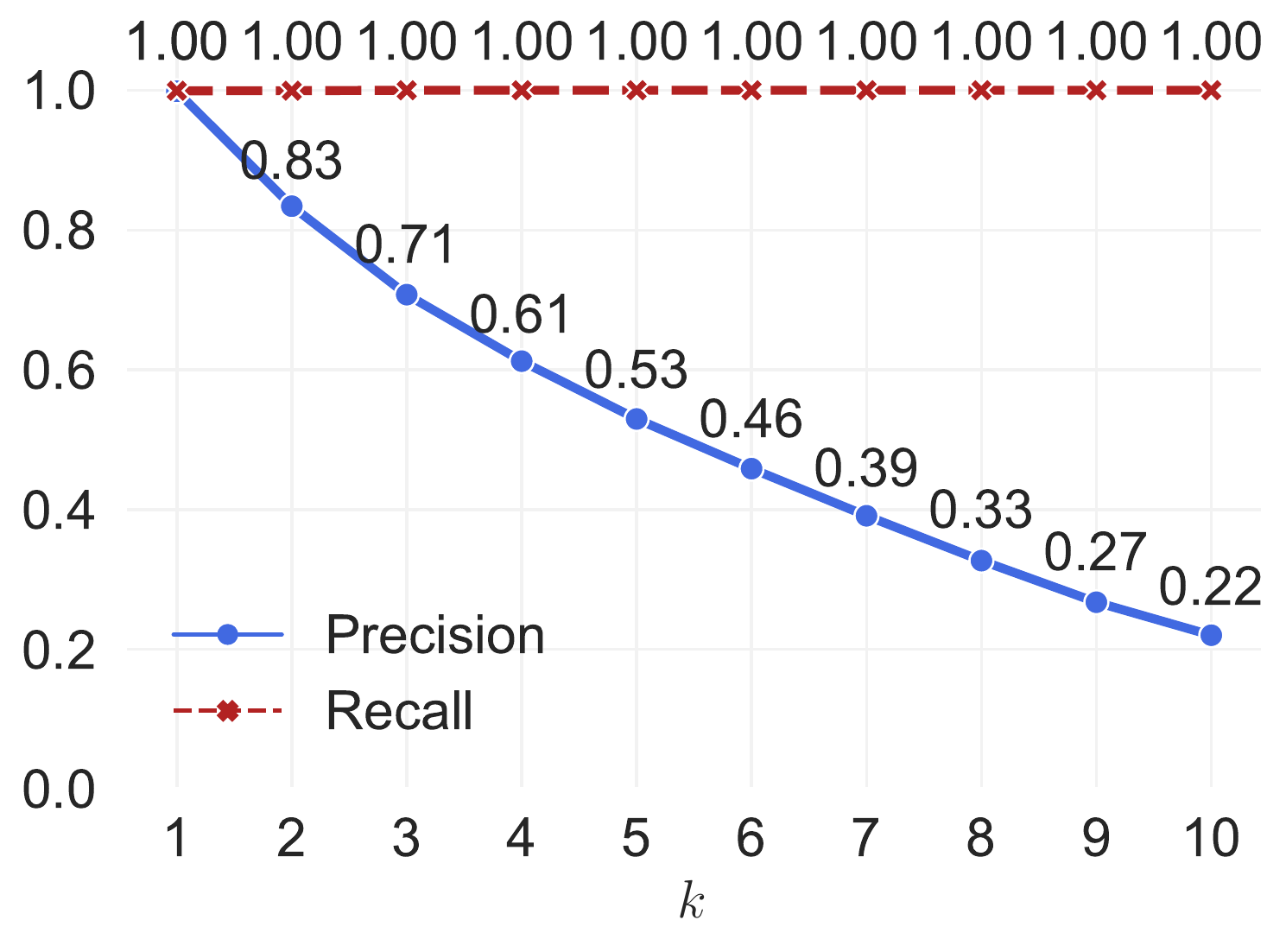}
\end{subfigure}
\begin{subfigure}{0.33\linewidth}
\centering
\includegraphics[width=\linewidth]{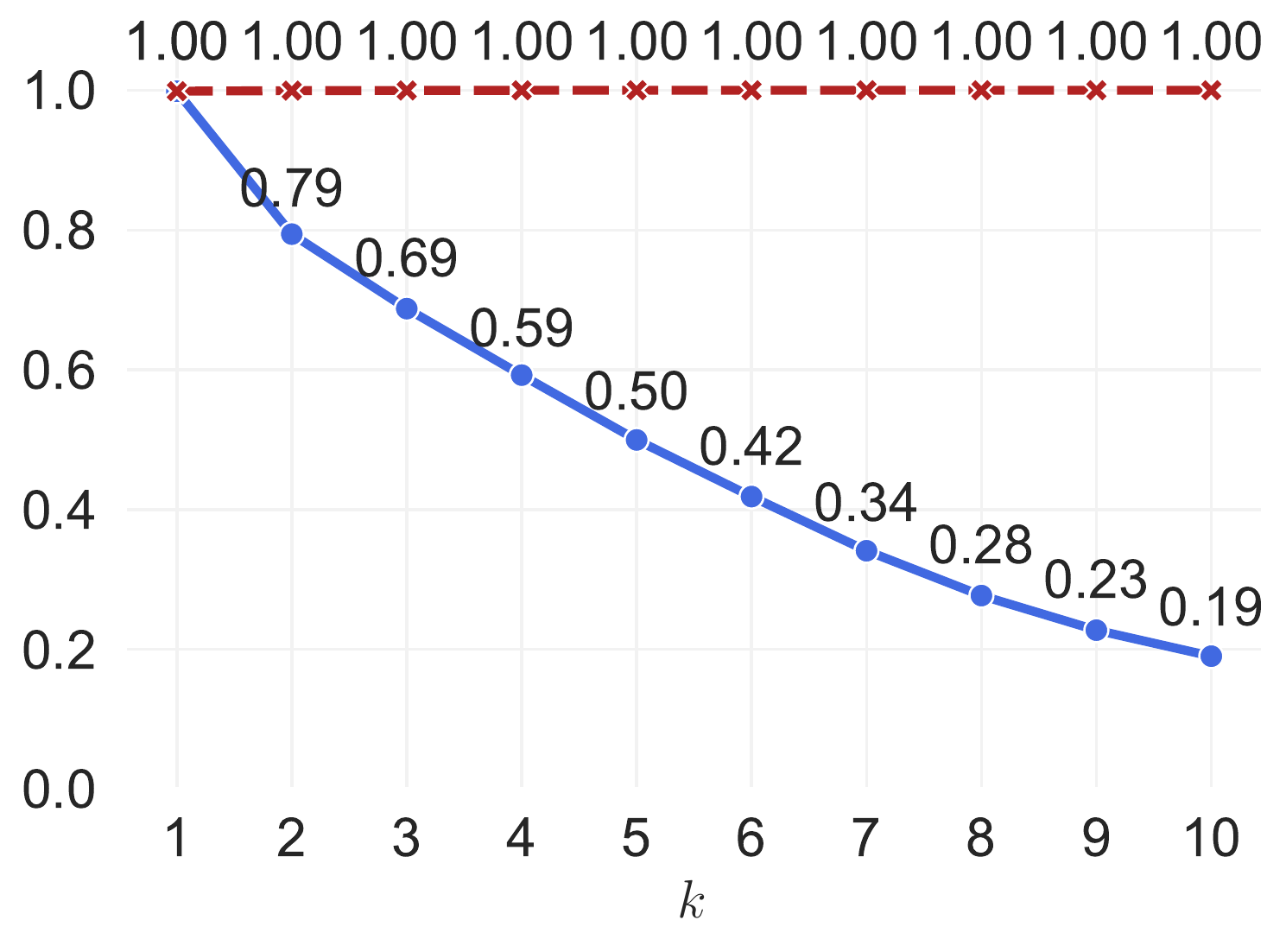}
\end{subfigure}
\begin{subfigure}{0.33\linewidth}
\centering
\includegraphics[width=\linewidth]{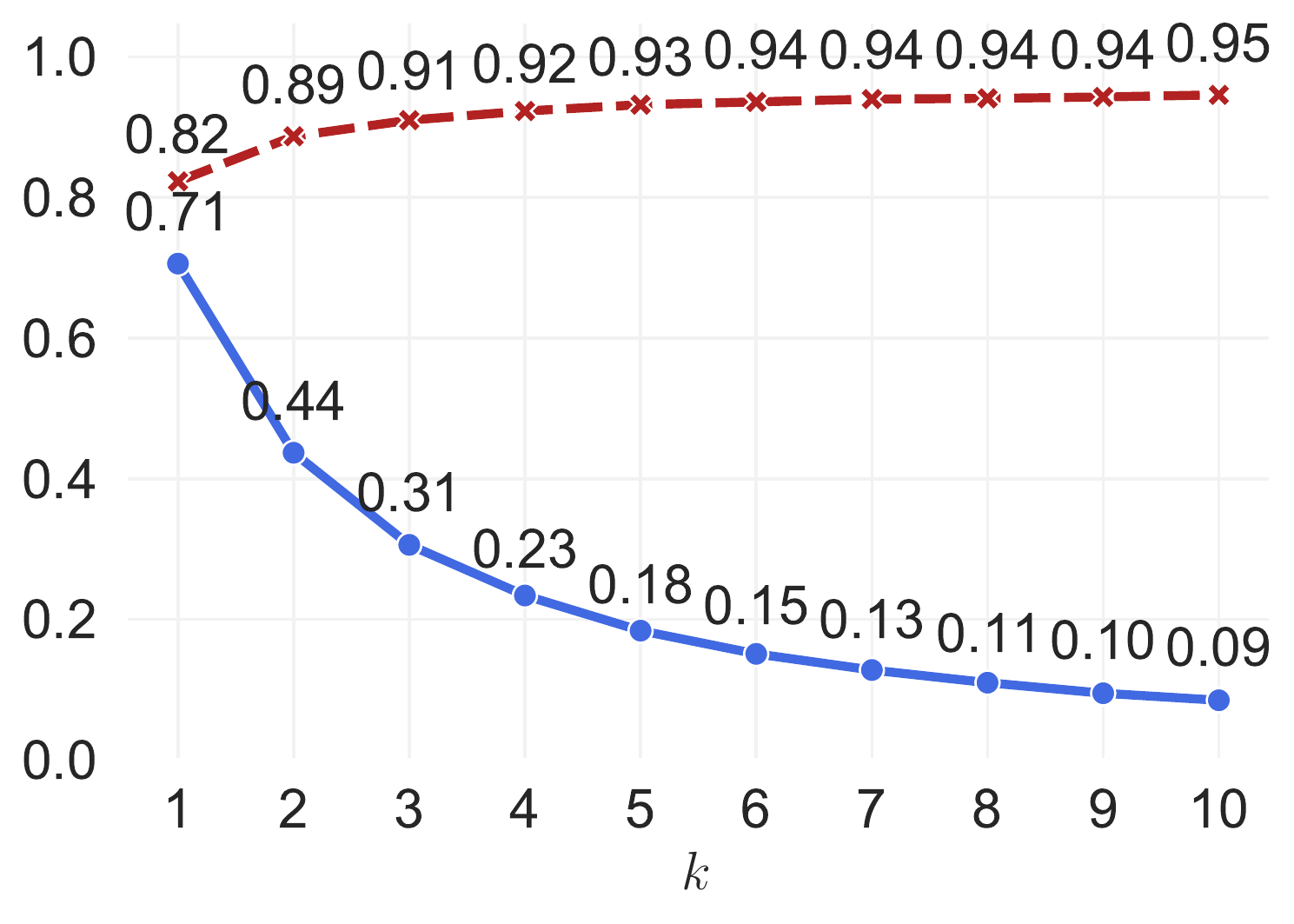}
\end{subfigure}
\caption{Precision and recall of the top $k$ sequences of relations on MetaQA 2 (left), MetaQA 3 (center) and WebQSP (right)}
\label{fig:paths_prec_recall}
\end{figure*}

\subsection{Efficiency and scalability}
\label{sec:scalability}

We analyze the efficiency of our approach on synthetic KBs (as in \citep{Cohen2017,Cohen2020}) and then compare the scalability of different preprocessing methods on the KGs of \textit{MetaQA} and \textit{WebQuestiontsSP}.
First, we perform experiments on KBs where the relational coalescing has no effect: the outdegree of each node is equal to the number of relation types and all edges originating from a node have different relation labels.
We perform two experiments on such KBs. In the first one (Figure \ref{fig:scalability_entities}), the number of relation types is fixed to $|\relations| = 10$ and the number of entities varies from $|\nodes| = 10^2$ to $|\nodes| = 10^6$. 
In the second task (Figure \ref{fig:scalability_relations}), the number of entities is fixed to $|\nodes| = 5000$ and the number of relations varies from $|\relations| = 1$ to $|\relations| = 10^3$.
The single answer node is always two-hops away from the entities mentioned in the question.
We compare \ours{} (unrefined) against a GNN-based approach (\textbf{GRAFT-Net} \cite{Sun2018}) and a key-value memory network (\textbf{KV-MemNN} \cite{Miller2016}). 
The approaches are evaluated based on the queries per second at inference time with a mini-batch size of 1.
The results show that increasing the number of entities has negligible impact on the performance of \ours{}, whereas GRAFT-Net and the key-value memory network are limited to graphs with less than 10k nodes.
This shows that, in large KGs like Freebase, the baselines would not be able to handle even a 2-hop neighborhood of the entities mentioned in the question (we refer the reader to Appendix \ref{app:scalability_coalescing} for more details).
Finally, from the results in Figure \ref{fig:scalability_relations}, we see that the throughput of our approach decreases with the number of relation types. 
However, in practice, we can leverage the GPU to score the edges of the graph in parallel. This is why we observe only a minor drop in performance when the number of relation types grows from $|\relations| = 1$ to $|\relations| = 100$.

In order to assess the scalability of the proposed relational coalescing operation, we further compare commonly used preprocessing methods on the KG of \textit{WebQuestionsSP}. We evaluate the time required to extract complete 2-hop neighborhoods of the entities mentioned in the question and the time to perform Personalized Page Rank (PPR) on such graphs. The results are shown in Figure \ref{fig:scalability_preprocessing}. Note that, at inference time, we can perform the coalescing only on the portion of the graph explored by the model, which makes \ours{} much more efficient. At training time, the preprocessing is comparable to the 2-hop neighborhood extraction. Finally, Figure \ref{fig:scalability_edges} shows the performance of the models with the respective preprocessing step at inference time on synthetic KBs with growing number of edges.

\begin{figure}[!tb]
\centering
\begin{subfigure}{0.245\linewidth}
\centering
\includegraphics[width=\linewidth]{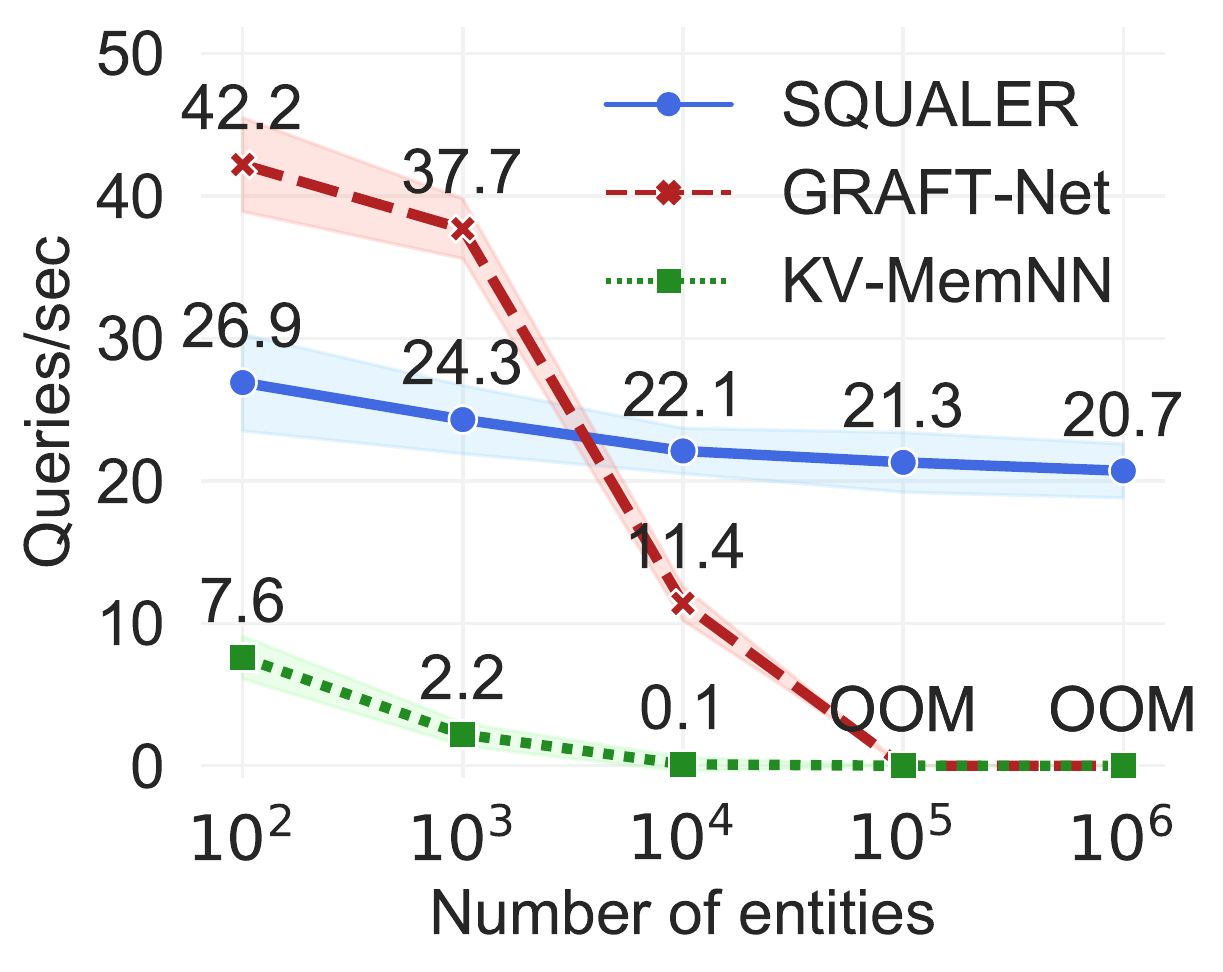}
\caption{}
\label{fig:scalability_entities}
\end{subfigure}
\begin{subfigure}{0.245\linewidth}
\centering
\includegraphics[width=\linewidth]{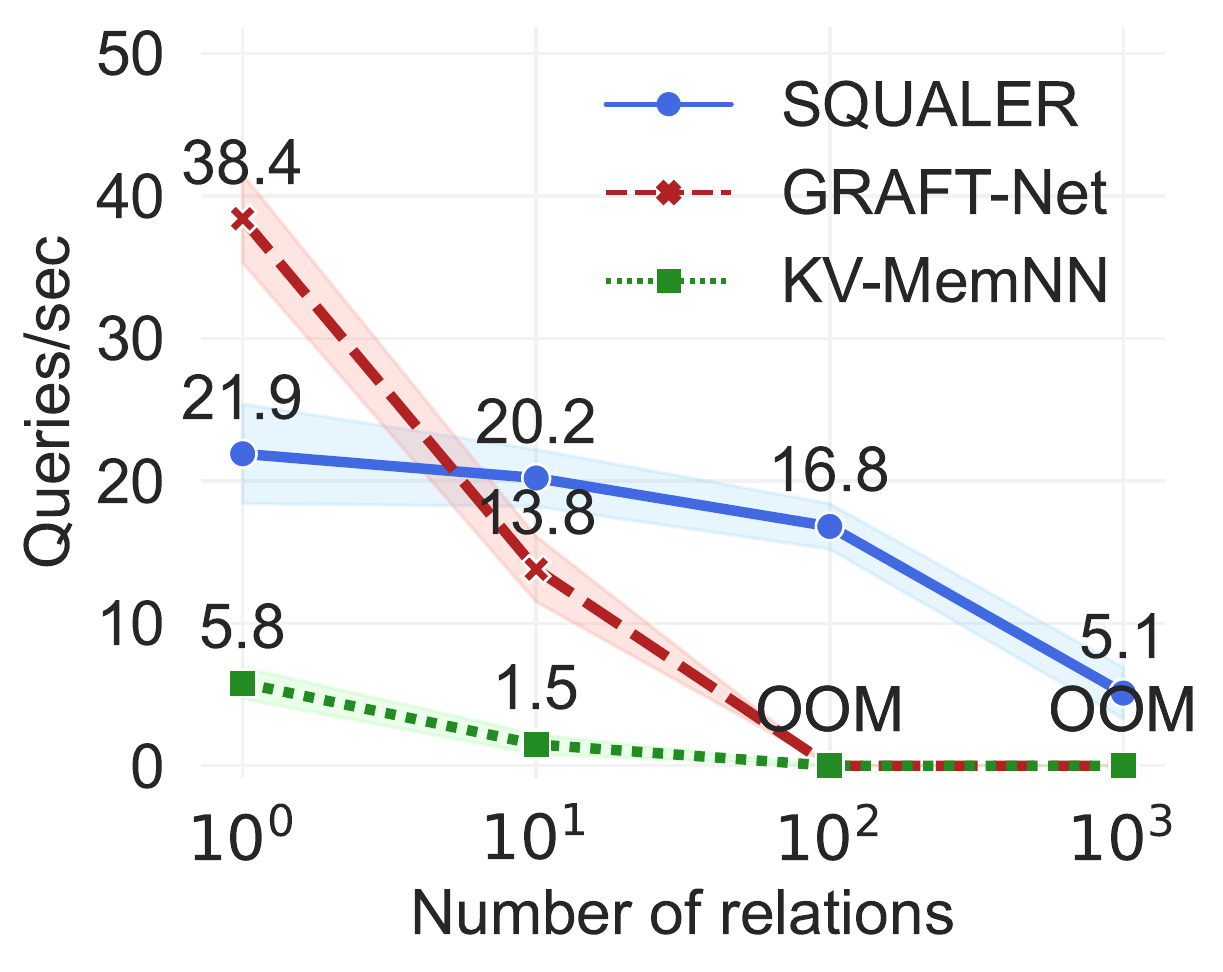}
\caption{}
\label{fig:scalability_relations}
\end{subfigure}
\begin{subfigure}{0.245\linewidth}
\centering
\includegraphics[width=\linewidth]{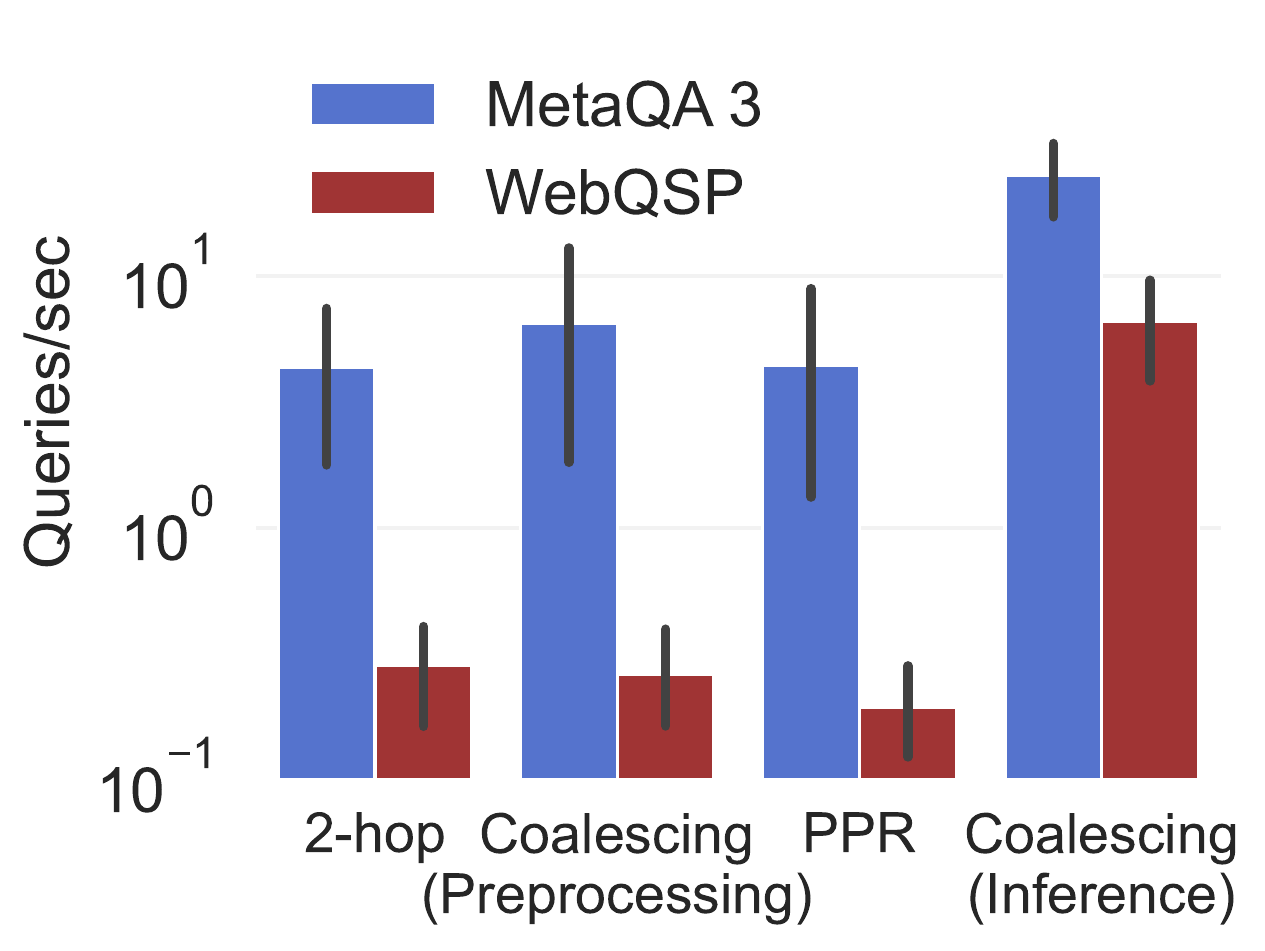}
\caption{}
\label{fig:scalability_preprocessing}
\end{subfigure}
\begin{subfigure}{0.245\linewidth}
\centering
\includegraphics[width=\linewidth]{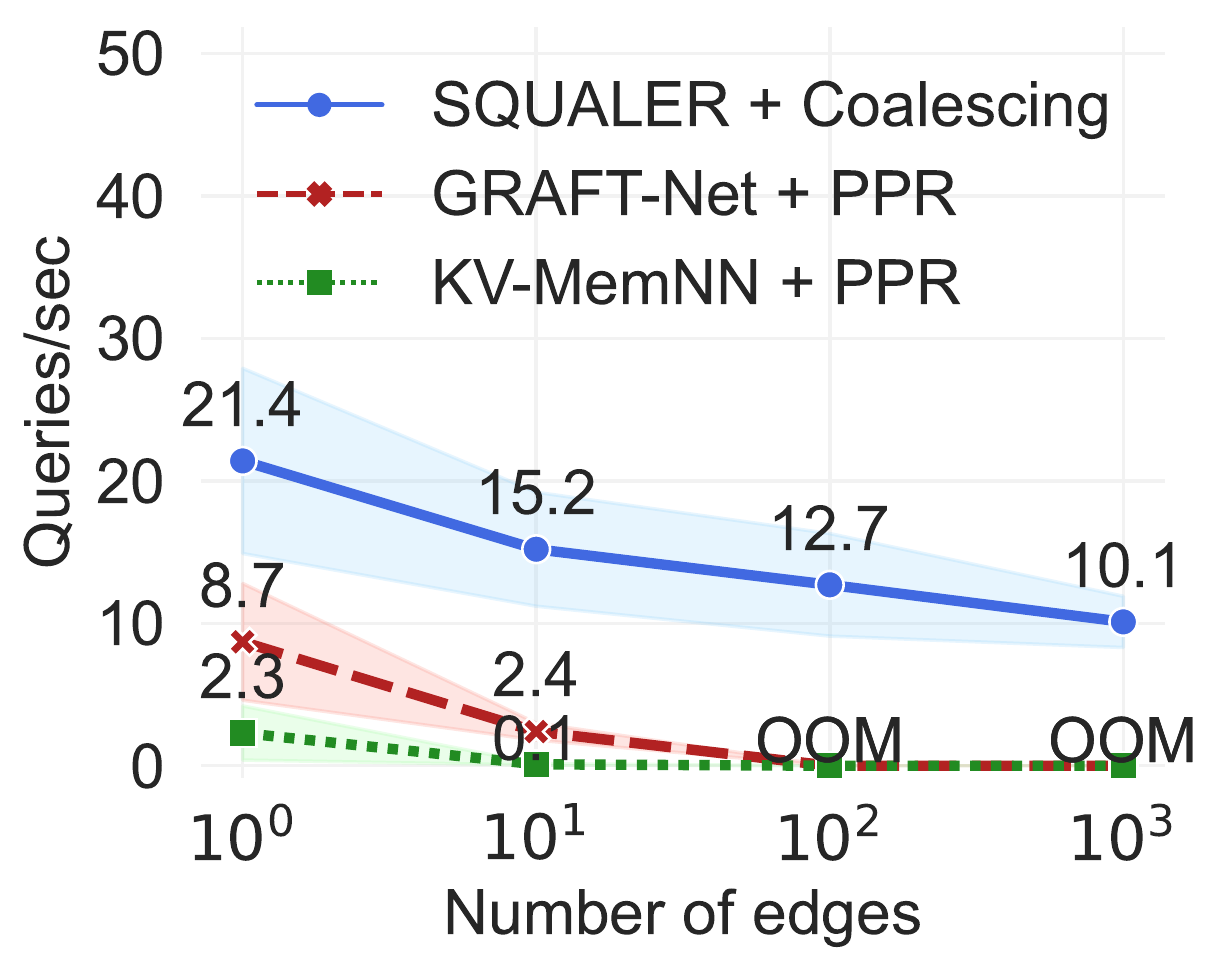}
\caption{}
\label{fig:scalability_edges}
\end{subfigure}
\caption{Inference time in queries/sec on synthetic KBs with increasing number of entities (a) and relation types (b). Time required by different preprocessing steps on the KG of \textit{WebQuestionsSP} and \textit{MetaQA} (c).  Complete inference and preprocessing time on synthetic KBs with increasing number of edges (d). We set the queries/sec to 0 when the model runs out-of-memory (OOM).}
\label{fig:scalability}
\end{figure}

\subsection{Incomplete knowledge graphs}
In order to evaluate the capability of our approach to cope with missing information in the knowledge graph, we performed two additional experiments. In the first experiment, we evaluated our approach (the \ours{} \textbf{-- GNN} variant) on \textit{WebQuestionsSP} using incomplete knowledge graphs with only 50\% of the original edges (\textbf{50\% KG}). Then, following previous work \citep{Sun2018,Sun2019}, we tried to mitigate the missing information using additional sources of external knowledge. In particular, for each question, we used the same text documents extracted from Wikipedia as done by \citet{Sun2018} (\textbf{50\% KG + Text}). In this experiment, the relation-level model is unaware of the additional source of knowledge, but the information from the text documents is infused into the edge-level GNN with the same strategy used in GRAFT-Net \citep{Sun2018} (note that this makes the edge-level GNN-based model essentially equivalent to the full version of GRAFT-Net, with both KG and text support).
We compare our approach against \textbf{GRAFT-Net} and \textbf{PullNet}, namely the two baselines designed for open-domain question answering with incomplete KGs and text documents.

The results of the experiments are reported in Table \ref{tab:incomplete}. We observe that, despite not being designed for incomplete KGs, \ours{} outperforms the baselines on both experimental settings.
This is not surprising, as \textbf{GRAFT-Net} relies on a simple heuristic process to construct question subgraphs and \textbf{PullNet} is constrained to follow the structure of the incomplete graph, because its iterative retrieval process can only expand nodes that are reachable from the set of anchor entities. This means that, in principle, any node retrieved by \textbf{PullNet}'s iterative process can also be reached by \ours{}’s relation-level model.
Similarly to the baselines, we note only a minor gain in performance when using the text documents as an additional source of information.

\begin{table}[!htb]
\caption{Hits@1 on \textit{WebQuestionsSP} with incomplete KGs (50\% of the edges) and additional text}
\centering
\label{tab:incomplete}
\resizebox{0.475\textwidth}{!}{%
\begin{tabular}{@{}lEE@{}}
\toprule
 & \textbf{50\% KG} & \textbf{50\% KG + Text} 
\EndTableHeader \\
\midrule
\textbf{GRAFT-Net} \citep{Sun2018} & 48.2 &	49.9 \\
\textbf{PullNet} \citep{Sun2019} & 50.3	& 51.9 \\
\textbf{\ours{} -- GNN} & 53.5 & 55.2 \\ \bottomrule
\end{tabular}%
}
\end{table}

\section{Related work}
Several lines of research in the past few years have focused on introducing deep learning approaches aimed at reasoning over structured knowledge.
In particular, this paper is closely related to methods for learning to traverse KGs \cite{Das2018,Das2017,Guu2015} and recent works on answering conjunctive queries using deep learning approaches \citep{Hamilton2018,Daza2020}.
In this context, several KB and query embedding methods have been proposed \citep{Wang2017}.
Many KB embedding approaches support the same operation performed by our relation-level model, namely relation projection \citep{Cohen2020,Sun2020,Hamilton2018,Ren2020}. Some KB embedding methods also explicitly learn to follow chains of relations and traverse KGs \citep{Guu2015,Lin2015,Das2017}.
Notably, Query2Box \citep{Ren2020} is a query embedding method that represents sets using box embeddings and the more recent beta embeddings \cite{Ren2020beta} extend the framework to support a complete set of first-order logic operators.
The main difference with our model is that these methods operate on vector space, whereas our approach is constrained on the graph structure and learns to traverse the KG while keeping the ability to scale to large graphs. Also, our method answers questions in natural language, while the above methods are primarily designed for query answering.
Recently, \citet{Sun2020} introduced EmQL, a query embedding method which has also been integrated in a question answering model.

Other lines of research on KBQA have focused on unsupervised semantic parsing \citep{atzeni2018what,atzeni2018translating,atzeni2018towards} or on the introduction of supervised models, like graph neural networks (GNNs) designed for reasoning over knowledge graphs \cite{Sun2018,Sun2019,yasunaga2021}.
These approaches pose the KBQA problem as a node classification task. For this reason, they have been applied succesfully only on small query-dependent graphs.
\citet{Cohen2020} addressed the problem of creating a representation of a symbolic KB that enables building neural KB inference modules that are scalable enough to perform non-trivial inferences with large graphs.
Another recent work \cite{Saxena2020} has explored using KG embeddings for question answering and handle incompleteness in the KG.
%


In our work, we combine relation projection with an edge-level GNN to address the KBQA problem. The same idea of combining GNNs with relational following was introduced in Gretel \citep{cordonnier2019}, which learns to complete natural paths in a graph given a path prefix.
Also, our idea of accelerating GNNs by operating on a reduced graph representation has strong connections with graph coarsening and sparsification \citep{loukas2019,loukas2018,Batson2013}.

Methods based on reinforcement learning (RL) have also been proposed to perform multi-hop reasoning over knowledge graphs.
\citet{Xiong2017} proposed DeepPath, which
relies on a policy-based agent that
learns to reason over multi-hop paths by sampling relations at each step.
Also, \citet{Das2018} introduced MINERVA, a RL agent that learns how to navigate the graph conditioned only on an input entity and on a query.
These approaches are designed for simple query answering and KB completion rather than KBQA.
A main difference with our work is that \ours{} samples multiple paths and employs an edge-level model to reach higher expressivity.

\section{Conclusion}
This paper introduced \ours{}, a scalable approach to reasoning and question answering over KGs.
Our method is expressive and can reach state-of-the-art performance on widely used and challenging datasets. Further, \ours{} scales with the number of (distinct) relation types in the graph and can effectively handle large-scale knowledge graphs with millions of entities. Our empirical evaluation also showed that our approach can generalize compositionally and that it can be used to generate question-dependent subgraphs that strike a good trade-off between precision and recall.

Overall, our work proposes an improvement to existing KBQA technology which carries impact to several practical applications. Nevertheless, we remind that the deployment of such models needs to be done cautiously. KBQA replaces a mature technology (traditional KBs and query languages) with less understood methods.
The underlying KB may be incomplete, contain misinformation or biases that could negatively affect the decisions of the learned model. We hope that our work will spur further research in this area and contribute to the development of reliable KBQA systems.

\acksection
Andreas Loukas would like to thank the Swiss National Science Foundation for supporting him in the context of the project “Deep Learning for Graph Structured Data”, grant number PZ00P2 179981. 

\nocite{Bordes2013,han2018openke,Loshchilov2019}

\bibliographystyle{plainnat}
\bibliography{references}

\newpage

\input{supplementary}

\end{document}

%% file: supplementary.tex
\newcommand{\toptitlebar}{
  \hrule height 4pt
  \vskip 0.25in
  \vskip -\parskip%
}
\newcommand{\bottomtitlebar}{
  \vskip 0.29in
  \vskip -\parskip
  \hrule height 1pt
  \vskip 0.09in%
}

\vbox{%
    \hsize\textwidth
    \linewidth\hsize
    \vskip 0.1in
    \toptitlebar
    \centering
    {\LARGE\bf Appendix\par}
    \bottomtitlebar
    \vskip 0.25in
  }

\appendix

\section{Formal definition of the coalesced representation}
\label{app:coalesced_representation}
Given a knowledge graph $\kg = (\nodes, \relations, \edges)$ and a set of entities $\seed$,
we can provide an alternative recursive definition of $\reach(\seed, \relseq)$ as:
\[
\reach(\seed, (r_1, r_2, \dots, r_{|\relseq|})) =
\begin{cases}
\seed \hspace{5mm} & \text{if } |\relseq| = 0  \vspace{1mm} \\
\reach(\seed', (r_2, \dots, r_{|\relseq|})) & \text{if } \seed \xrightarrow{r_1} \seed' \\
\emptyset & \text{otherwise}
\end{cases}
\]
where $\seed'$ is the set of nodes reachable from $\seed$ by an $r_1$ relation.

Then, we can define the coalesced representation
$
\coalkg = (\coalnodes, \coalrel, \coaledges)
$
as follows:
\begin{itemize}
\item $\coalnodes = \{\reach(\seed, R) \mid R \in \relations^* \}$ are the nodes $R$-reachable from $\seed$ by $\relseq \in \relations^*$ ($*$ is the Kleene star);
\item $\coalrel = \relations \cup \{\mathtt{self}\}$ is the original set of relations augmented with the self-loop relation type $\mathtt{self}$, which denotes the empty sequence $\mathtt{self} \in \relations^*$;
\item edge $\nodes_i \xrightarrow{r} \nodes_j$ belongs to $\coaledges$ if and only if $\nodes_j = \reach(\nodes_i, r)$, with $r \in \coalrel$.
\end{itemize}
Intuitively, this operation can be seen as coalescing relations in the original knowledge graph $\kg$ and adding self loops.
In practice, we do not need to compute all the nodes in $\coalkg$ but only edge labels.

\section{Computational Complexity}
\label{app:computational_complexity}

The knowledge seeking procedure described in Section \ref{sec:seeking} applies a search algorithm over the graph $\coalkg$ to obtain the most likely set of relation sequences originating from $\seed$.
The exact knowledge seeking procedure adopted in our experiments is based on the beam search algorithm and is detailed in Algorithm \ref{alg:knowledge_seeking}.
The algorithm is designed to scale with the number of relation types in the original knowledge graph, which is usually much smaller than the number of edges (facts) or nodes (entities). In this section, we describe the algorithm in more details and we provide an extensive analysis of the computational complexity of our approach.

\begin{algorithm}[!b]
\caption{Knowledge Seeking}
\label{alg:knowledge_seeking}
\SetKwInOut{Input}{Input}
\SetKwInOut{Output}{Output}
\SetKw{KwAnd}{and}
\SetKw{KwOr}{or}
\DontPrintSemicolon 

\Input{a coalesced knowledge graph $\coalkg$; a set of starting entities $\seed$; the beam width $\beta$; the maximum number of iterations $\tmax$; and the number of relation sequences to be returned $k \leq \beta$}
\Output{A set of $k$ tuples of the form $(\candidates, \relseq, w)$, representing the $k$ most likely candidate answers $\candidates$, the sequence of relations $\relseq$ to reach $\candidates$, and the negative log-likelihood $w$ of $\relseq$}

\algnewcommand{\IfThenElse}[3]{
  \algorithmicif\ #1\ \algorithmicthen\ #2\ \algorithmicelse\ #3}

\medskip
$t \gets 1$\;
$\beams_t \gets \{(\seed, \mathtt{self}, 0)\}$\;
\medskip
\Repeat{$\beams_t = \beams_{t - 1}$ \KwOr $t > \tmax$}{
    $\beams_{t + 1} \gets \emptyset$\;
    \For{$(\nodes_t, \relseq_t, w_t) \in \beams_t$}{
        \uIf{$\relseq_t = (r_0, \dots, \mathtt{self})$ \KwAnd $t > 1$}{
        $\beams_{t + 1} \gets \beams_{t + 1} \cup \{(\nodes_t, \relseq_t, w_t)\}$\;}
        \Else{
        $\coal{\edges}_t \gets \{\nodes_t \xrightarrow{r_t}{\nodes_{t + 1}} \in \coaledges\}$\;
        \For{$\nodes_t \xrightarrow{r_t} \nodes_{t + 1} \in \coal{\edges}_t$}{
            $\relseq_{t + 1} = (\relseq_t, r_t)$\;
            $w_{t + 1} \gets w_t - \log \relweights(\nodes_t \xrightarrow{r_t} \nodes_{t + 1})$\;
            $\beams_{t + 1} \gets \beams_{t + 1} \cup \{(\nodes_{t+1}, \relseq_{t+1}, w_{t+1})\}$\;
        }
        }
    }
     $\beams_{t + 1} \gets \mathsf{min}(\beams_{t + 1}, \beta)$\;
        $t \gets t + 1$
    
}

\medskip

\Return $\mathsf{min}(\beams_t, k)$

\end{algorithm}

\paragraph{Overview of the knowledge seeking procedure}
At each iteration, Algorithm \ref{alg:knowledge_seeking} updates a set $\beams_t$ containing triples of the form $(\nodes_t, \relseq_t, w_t)$. We denote with $\nodes_t = \reach(\seed, \relseq_t)$ the set of nodes reachable from $\seed$ by following $\relseq_t$, whereas $\relseq_t$ represents a relation sequence constructed iteratively by applying the relation-level model on edges of $\coalkg$ up to time step $t$. The last element of the tuples $w_t$ is the total accumulated negative log-likelihood of $\relseq_t$, computed as explained in Section \ref{sec:seeking}.
At the beginning of the algorithm, $\nodes_1 = \seed$ is the set of entities mentioned in the natural language question, $\relseq_1 = \mathtt{self}$ is the empty relation sequence and we set the initial negative log-likelihood $w_1 = 0$.
The algorithm receives as input a parameter $\beta$ which specifies the \textit{beam width}, namely the number of relation sequences that are expanded at each iteration.
At time step $t$, we compute the set $\coal{\edges}_t$ of all edges originating from $\nodes_t$ in $\coalkg$.
Then, the relation sequences $R_t$ are expanded with the relation types labeling edges in $\coal{\edges}_t$. The likelihood of the new relation sequences is calculated based on $w_t$ and the likelihood assigned by the relation-level model to the relation type appended to $\relseq_t$.
At the end of each iteration, the function $\mathsf{min}(\beams_{t + 1}, \beta)$ in Algorithm \ref{alg:knowledge_seeking} retains for the next time step only the $\beta$ tuples $(\nodes_{t + 1}, \relseq_{t + 1}, w_{t + 1}) \in \beams_{t + 1}$ with the minimum negative log-likelihood $w_{t + 1}$.
Note that, in Algorithm \ref{alg:knowledge_seeking}, relation sequences ending with the $\mathtt{self}$ relation type are not expanded after the first time step. As explained in Section \ref{sec:model}, indeed, the $\mathtt{self}$ relation type is used to signal both the start and the end of the decoding.

\paragraph{Time complexity}
At time step $t$, for each triple $(\nodes_t, \relseq_t, w_t) \in \beams_t$, the algorithm computes $\relweights$ for all edges $\coal{\edges}_t$ originating from $\nodes_t$. 
This means that the relation-level model described in Section \ref{sec:model} is queried $|\beams_t| \cdot |\coal{\edges}_t|$ times at iteration $t$.
Note that we do not need to compute the likelihood $\relweights(\nodes_i \xrightarrow{r} \nodes_j)$ for all edges $\nodes_i \xrightarrow{r} \nodes_j$ in $\coaledges$.
Let $\maxdegout(\coalkg)$ be the maximum outdegree of nodes in $\coalkg$.
At time step $t$, the size of the set $\beams_{t+1}$ is restricted to $\beta$ for the next iteration by the operation $\mathsf{min}(\beams_{t + 1}, \beta)$. Since $|\beams_t|$ is bounded by $\beta$ and $|\coal{\edges}_t|$ is bounded by $\maxdegout(\coalkg)$, at any iteration, the relation-level model is queried at most $\beta \cdot \maxdegout(\coalkg)$ times. Each of such queries takes constant time.
The function $\mathsf{min}(\beams_{t+1}, \beta)$ selects the $\beta$ tuples in $\beams_{t + 1}$ with the smallest negative log-likelihood. This can be done on average in $\mathcal{O}(|\beams_{t + 1}|)$ time.
At iteration $t$, the set $\beams_{t+1}$ is initialized as the empty set and updated by adding at most $\beta \cdot \maxdegout(\coalkg)$ tuples (one element for each query to the relation-level model). Therefore, the expected time complexity of the function $\mathsf{min}(\beams_{t+1}, \beta)$ is $\mathcal{O}(\beta \cdot \maxdegout(\coalkg))$.
Now, note that by the definition of $\coalkg$, we have $\maxdegout(\coalkg) \leq |\coalrel| = |\relations| + 1$. Hence, the number of queries to the relation-level model is bounded by $\beta \cdot (|\relations| + 1)$ and the time complexity of $\mathsf{min}(\beams_{t+1}, \beta)$ is also $\mathcal{O}(\beta \cdot |\relations|)$.
The maximum depth reached by the knowledge seeking procedure starting from $\seed$ is bounded by $\tmax$, because Algorithm \ref{alg:knowledge_seeking} performs at most $\tmax$ iterations of the main outer loop.
The final step $\mathsf{min}(\beams_t, k)$ selects the $k$ most likely tuples and can be run on average in $\mathcal{O}(\beta)$ time.
This yields a final computational complexity of
\[
\mathcal{O}(\tmax \cdot \beta \cdot |\relations|) = \mathcal{O}(|\relations|).
\]
Note that $\tmax$ and $\beta$ are constant parameters of the algorithm and are usually small. In our experiments, we set $\tmax = 3$ for MetaQA 3 and $\tmax = 2$ for MetaQA 2 and WebQSP. We set the beam width $\beta = 10$, obtaining only minor improvements with respect to a greedy search with $\beta = 1$.
Therefore, we obtain that that time complexity of the knowledge seeking procedure scales linearly with the number of relation types and does not depend on the number of nodes or edges in $\kg$. 

\paragraph{Space complexity}
For each iteration $t$, Algorithm \ref{alg:knowledge_seeking} constructs $\beams_{t + 1}$ by analyzing all edges originating from each node $\nodes_t$ stored in the tuples $(\nodes_t, \relseq_t, w_t) \in \beams_t$. From the considerations reported above, the size of $\beams_{t + 1}$ is $\mathcal{O}(\beta \cdot |\relations|)$. Although for notational convenience we are representing $\beams_t$ as a set of triples, in practice we can avoid storing intermediate nodes $\nodes_t$ and construct the set of candidate answers by following $\relseq_t$ at the final iteration. Therefore, we only need to store relation sequences $\relseq_t$ and their negative log-likelihood $w_t$. Each tuple requires $\mathcal{O}(\tmax)$ space, as $|\relseq_t|$ is bounded by $\tmax$. The space complexity of the algorithm is thus $\mathcal{O}(\tmax \cdot \beta \cdot |\relations|)$.

\section{Expressive Power}
\label{app:expressive_power}
As mentioned in Section \ref{sec:analysis}, the approach described in this paper can be used to answer any valid \textit{existential positive first order query} on a knowledge graph $\kg$.
In order to prove this, we first consider the simpler class of \emph{conjunctive queries}. We will show a result similar to Proposition \ref{prop:expressive_power} for conjunctive queries, and then we will extend this result to the wider class of EPFO queries.

\subsection{Conjunctive Queries}
Given a knowledge graph $\kg = (\nodes, \relations, \edges)$ and a non-empty set of nodes $\seed \subseteq \nodes$, a conjunctive query on $\kg$ is a query involving only existential quantification and conjunction:

\[
\lquery[\lvar_{?}] =  \lvar_{?} . \exists \lvar_1, \dots, \lvar_{m} : \llit_1 \land \llit_2 \land \dots \land \llit_{|\lquery|},
\]

such that each literal $\llit_{i}$ is an atomic formula of the form $\lrel(\lvar, \lvar')$, where $\lvar \in \seed \cup \{\lvar_1, \dots, \lvar_m\}$, $\lvar' \in \{\ltarget, \lvar_1, \dots, \lvar_m\}$, $\lvar \neq \lvar'$, and $\lrel(\lvar, \lvar')$ is satisfied if and only if $\lvar \xrightarrow{r} \lvar'$, $r \in \relations$. 

In general, for any query $\lquery$, we can define its \textit{dependency graph} as the graph with nodes $\seed \cup \{\ltarget, \lvar_1, \dots, \lvar_m\}$. The edges of the graph are the literals $\{\llit_1, \dots, \llit_{|\lquery|}\}$, as each literal is of the form $\lrel(\lvar, \lvar')$ and defines an edge between $\lvar$ and $\lvar'$ \cite{Hamilton2018}. Figure \ref{fig:computation_graph} shows an example of the dependency graph of a conjunctive query.

We say that a query is \textit{valid} if its dependency graph is a directed acyclic graph (DAG), with $\seed$ as the source nodes and the target variable $\ltarget$ as the unique sink node. In the following, we will always consider valid queries, as this ensures that the query has no redundancies or contradictions.

\begin{lemma}
\label{lemma:conjuctive_queries}
Let $\kg = (\nodes, \relations, \edges)$ be a knowledge graph and $\lquery$ be a valid conjunctive query on $\kg$.
Then, there exists a sequence of relations $\relseq^{\star} \in \relations^*$ such that:
\[
\answers \subseteq \reach(\seed, \relseq^{\star}),
\]
where $\answers = \{v \in \nodes \mid \lquery[v] = \mathsf{True}\}$ is the denotation set of $\lquery$, namely the entities satisfying $\lquery$.
\end{lemma}

\begin{proof}
We proceed by induction on the number of literals $|\lquery|$.

\textit{Base case.}
Assume $|\lquery| = 1$. Then, since $\lquery$ is valid, the query is of the form:
\[
\lquery[\lvar_{?}] =  \ltarget . \lrel(v, \ltarget),
\]
with $\{v\} = \seed$.
We have:
\begin{align}
\answers &= \{ v' \in \nodes \mid \lquery[v'] = \mathsf{True} \} \\
&= \{ v' \in \nodes \mid v \xrightarrow[]{r} v'\} \\
&= \reach(\seed, r).
\end{align}
Hence the sequence with only relation $r$ is sufficient to generate the set of correct answers $\answers$.

\textit{Inductive step.}
Let $\lquery$ be a conjunctive query of the form:
\[
\lquery[\lvar_{?}] =  \lvar_{?} . \exists \lvar_1, \dots, \lvar_{m} : \llit_1 \land \llit_2 \land \dots \land \llit_{|\lquery|}.
\]
Assume that there exists a sequence of relations $\relseq^{\star} \in \relations^*$, such that:
\[
\answers = \{v \in \nodes \mid \lquery[v] = \mathsf{True}\} \subseteq \reach(\seed, \relseq^{\star}).
\]
Consider a query $\lquery'$ constructed by adding a literal $\llit_{|\lquery| + 1}$ to $\lquery$, and let $\answers'$ be the denotation set of $\lquery'$, namely the set of nodes satisfying $\lquery'$. The conjunctive query $\lquery'$ may or may not have the same target variable of $\lquery$.

If $\lquery'$ shares the same target variable of $\lquery$, then $\lquery'$ is of the form:
\[
\lquery'[\lvar_{?}] =  \lvar_{?} . \exists \lvar_1, \dots, \lvar_{m} : \llit_1 \land \llit_2 \land \dots \land \llit_{|\lquery|} \land \llit_{|\lquery + 1|}.
\]
Note that:
\begin{align}
\answers' &= \{v \in \nodes \mid \lquery'[v] = \mathsf{True}\} \\
& \subseteq \{v \in \nodes \mid \lquery[v] = \mathsf{True}\} \\
& \subseteq \reach(\seed, \relseq^{\star}).
\end{align}
Hence, if $\lquery$ and $\lquery'$ share the same target variable, the same sequence of relations that generates candidate answers for $\lquery$ can be used to generate candidate answers for $\lquery'$.

If $\lquery$ and $\lquery'$ do not have the same target variable, then we can write $\lquery'$ as:
\[
\lquery'[\ltarget'] =  \ltarget' . \exists \ltarget, \lvar_1, \dots, \lvar_{m} : \llit_1 \land \llit_2 \land \dots \land \llit_{|\lquery|} \land \llit_{|\lquery + 1|}.
\]
Since $\lquery'$ is a \emph{valid} conjunctive query on $\kg$, $\llit_{|\lquery + 1|}$ is of the form $\lrel(\ltarget, \ltarget')$. Then we have that:
\begin{align}
\answers' &= \{v' \in \nodes \mid \lquery'[v'] = \mathsf{True}\} \\
&= \{v' \in \nodes \mid \exists \ltarget, \lvar_1, \dots, \lvar_{m} : \llit_1 \land \llit_2 \land \dots \land \llit_{|\lquery|} \land \ltarget \xrightarrow{r} v' \} \\
&= \{ v' \in \nodes \mid \exists v \in \answers : v \xrightarrow{r} v' \} \\
&= \reach(\answers, r) \\
& \subseteq \reach(\reach(\seed, \relseq^{\star}), r) \\
&= \reach(\seed, (\relseq^{\star}, r)).
\end{align}
Therefore, the sequence $(\relseq^{\star}, r)$ can be used to generate the answers to $\lquery'$.
\end{proof}

\subsection{Existential Positive First-Order Queries}
Any EPFO query can be expressed in disjunctive normal form (DNF), namely a disjunction of one or more conjunctions:

\[
\lquery[\lvar_{?}] =  \ltarget . \exists \lvar_1, \dots, \lvar_{m} : \lconj_1 \lor \lconj_2 \lor \dots \lor \lconj_{n_\lor + 1},
\]
such that:
\begin{itemize}
\item each $\lconj_i$ is a conjunction of literals of the form $\lconj_i = \llit_{i1} \land \llit_{i2} \land \dots \land \llit_{i|\lconj_i|}$
\item each literal $\llit_{ij}$ is an atomic formula of the form $\lrel(\lvar, \lvar')$, where $\lvar \in \seed \cup \{\lvar_1, \dots, \lvar_m\}$, $\lvar' \in \{\lvar_{?}, \lvar_1, \dots, \lvar_m\}$, $\lvar \neq \lvar'$, and $\lrel(\lvar, \lvar') = \mathsf{True}$ if and only if $\lvar \xrightarrow{r} \lvar'$, $r \in \relations$.
\end{itemize}

As above, we assume that the $\lquery$ is a \emph{valid} query on $\kg$, namely all $\lconj_i$ are valid conjunctive queries.
As shown in Figure \ref{fig:computation_graph}, we can represent any EPFO query $\lquery$ with a computation graph containing the operations that are required to answer $\lquery$. Specifically, each atomic formula can be represented as a relation projection, whereas conjunctions and disjunctions can be represented as intersection and union operations respectively.

\paragraph{Proof of Proposition \ref{prop:expressive_power}}
We assume that $\lquery$ is expressed in disjuctive normal form and we denote with $n_\lor$ the number of disjunction ($\lor$) operators in $\lquery$. We proceed by induction on $n_\lor$.

\emph{Base case}. Assume $n_\lor = 0$. Then, $\lquery$ is a conjunctive query, and by Lemma \ref{lemma:conjuctive_queries}, there exists $\relseq^{\star} \in \relations^*$ such that:
\[
\answers = \{ v \in \nodes \mid \lquery[v] = \mathsf{True} \} \subseteq \reach(\seed, \relseq^{\star}).
\]

\emph{Inductive step}.
Let $\lquery$ be an EPFO query in DNF:
\[
\lquery[\lvar_{?}] =  \ltarget . \exists \lvar_1, \dots, \lvar_{m} : \lconj_1 \lor \lconj_2 \lor \dots \lor \lconj_{n_\lor + 1}.
\]
Consider the subquery $\lquery'$ consisting of the conjunction terms $\lconj_1 \lor \lconj_2 \lor \dots \lor \lconj_{n_\lor}$ and assume that there exist $k \leq n_\lor$ sequences of relations $\relseq^{\star}_i$ such that:
\[
\answers' = \{v \in \nodes \mid \lquery'[v] = \mathsf{True} \} \subseteq \bigcup_{i = 1}^{k} \reach(\seed, \relseq^{\star}_i).
\]
Note that $\lconj_{n_\lor + 1}$ is a valid conjunctive query and by Lemma \ref{lemma:conjuctive_queries} there exists $\relseq^{\star}_{k + 1} \in \relations^*$ such that:
\[
\{v \in \nodes \mid \lconj_{n_\lor + 1}[v]\} \subseteq \reach(\seed, \relseq^{\star}_{k + 1}).
\]
Then, it holds that:
\begin{align}
\answers &= \{v \in \nodes \mid \lquery[v] = \mathsf{True} \} \\
&= \{v \in \nodes \mid \lquery'[v] \lor \lconj_{n_\lor + 1}[v]\} \\
&= \answers' \cup  \{v \in \nodes \mid \lconj_{n_\lor + 1}[v]\} \\
&\subseteq \bigcup_{i = 1}^{k} \reach(\seed, \relseq^{\star}_i) \cup \{v \in \nodes \mid \lconj_{n_\lor + 1}[v]\} \\
& \subseteq \bigcup_{i = 1}^{k + 1} \reach(\seed, \relseq^{\star}_i).
\end{align}
\hfill \qedsymbol

\begin{figure}[!t]
    \centering
    \includegraphics[width=\linewidth]{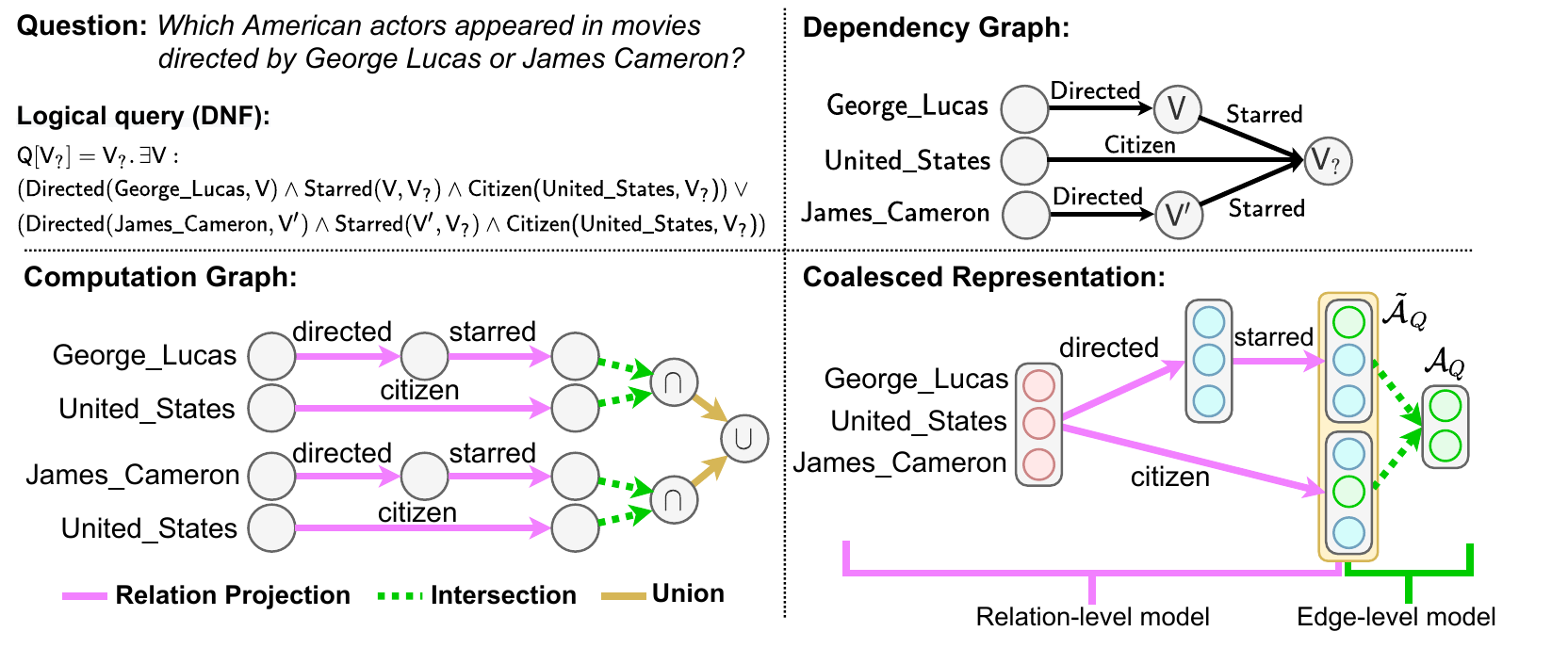}
    \caption{Example of a natural language question and the corresponding EPFO query expressed in DNF (top left); the dependency graph of the EPFO query (top right); the computation needed to answer the query in the original KG (bottom left); and the computation performed by our approach in the coalesced representation (bottom right). Note that, for completeness, we represent two paths in the coalesced representation, but only one is sufficient.}
    \label{fig:computation_graph}
\end{figure}

\section{Training strategies}
\label{app:training}
In this section we describe the training strategies that we used to optimize the parameters of our relation-level model and improve generalization performance.

\paragraph{Supervision}
For the experiments on KBQA, we assume that we only have access to pairs of questions and answers, i.e. the actual inferential chain leading from the question to the answer is latent.
Therefore, we resort to weak supervision to train the model.
Since at training time the set $\answers$ is known, we can compute all relation sequences $\relseq^{\star}$, such that $\candidates = \reach(\seed, \relseq^{\star})$ is the smallest reachable superset of $\seed$.
If the smallest reachable superset of $\answers$ is not unique, all relation sequences leading to any superset of $\answers$ of minimum cardinality are considered.
Note that the set of all possible relation sequences of a given length originating from $\seed$ in $\coalkg$ is much smaller than the set of all possible paths starting from nodes in $\seed$ in $\kg$, as shown in Appendix \ref{app:scalability_coalescing}.
Since the \textit{CFQ} dataset contains boolean questions (where the answer is not a set of entities), for the experiment on compositional generalization we use the logical parsing provided in the dataset to compute the correct sequences of relations.
We assume these sequences of relations are stored in such a way that the set of relations exiting from the a node in $\coalkg$ can be accessed efficiently in constant time. 
Then, at any decoding time step $t$, an edge is labeled as positive if and only if it belongs to a sequence of relations leading to $\candidates$.
The model is then trained using teacher forcing, namely we feed into the decoder relation sequences leading from $\seed$ to $\candidates$. We do not have multiple decoding time steps at training time, as the whole sequence is provided at once, and relation types are appropriately masked so that they cannot attend to items in future positions.


\paragraph{Path dropout}
Previous work \cite{Sun2018} has shown that randomly removing facts from the knowledge base at training time can be beneficial for generalization. Inspired by such insight, we employ a similar technique to enhance the performance of our model.
Specifically, in the first epochs, we randomly remove paths that are not labeled as correct with probability $p_{\textit{drop}}$, in order to make the problem easier for the model. This probability is the linearly decreased to 0 during training.
We set the initial $p_{\textit{drop}}$ to 0.5 and we gradually lower it to 0 until half of the training epochs have been run.

\paragraph{Pretraining and fine tuning}
For the experiments on \textit{WebQuestionsSP} and \textit{CFQ}, we found it beneficial to pretrain our model in order to incorporate knowledge from Freebase into the layers of the decoder.
Specifically, we sampled a total of 500k 1-hop or 2-hop paths and we trained the model to predict the sequence of relations connecting two nodes, given the embeddings of the source node and the target node of the path.
In order to do this, we replace the encoder with a simple 2-layer feed-forward network, with a ReLU non-linearity. This network receives as input two 100-dimensional embeddings for the source and target nodes of the path, and maps them to a $d_{\textit{model}}$-dimensional representation. This representation is then fed into the decoder to predict the relations connecting the two nodes.
We use a concatenation of 50-dimensional random and 50-dimensional pretrained TransE \cite{Bordes2013} embeddings \cite{han2018openke} to represent the entities in the KG.
Moreover, on \textit{WebQuestionsSP} we observed that it was helpful to fine-tune BERT in order to produce better representations of the relations in the knowledge graph. The same BERT model is still used to encode both the questions and relation types.

\section{Experimental Details}
\label{app:experimental_details}

\subsection{Datasets}
\label{app:datasets}

\paragraph{KBQA datasets}
We performed our experiments on KBQA on two widely adopted datasets, namely \textit{MetaQA} \citep{Zhang2018} and \emph{WebQuestionsSP} \citep{yih2015}. We provide below a detailed description of each one.
\begin{itemize}
\item \textbf{MetaQA}\footnote{\url{https://github.com/yuyuz/MetaQA}} \citep{Zhang2018} is a multi-hop question answering dataset including 400K question-answer pairs. Questions are answerable using the WikiMovies knowledge base.
The dataset includes 1-hop, 2-hop and 3-hop questions. It is provided under the Creative Commons Public License Attribution 3.0 Unported\footnote{\url{https://creativecommons.org/licenses/by/3.0/}}.
We evaluated our approach on 2-hop (\textbf{MetaQA 2}) and 3-hop (\textbf{MetaQA 3}) questions.
\item \textbf{WebQuestionsSP} \citep{yih2015} comprises 4737 questions over a subset of Freebase, which is provided under the CC BY 2.5 license\footnote{\url{https://creativecommons.org/licenses/by/2.5/}}. The questions in this dataset are answerable by performing relational following for up to two hops and an optional relational filtering operation on the result.
\end{itemize}

Table \ref{tab:data_stats} shows the number of questions in the training, development and test splits of each dataset. We use the same splits as in \citep{Sun2018}. Table 3 reports instead the number of triples (edges), entities (nodes) and relations in the KGs used in our experiments.

\begin{table}[!htb]
\caption{Number of questions in the training, development and test sets}
\centering
\label{tab:data_stats}
\begin{tabular}{@{}lccc@{}}
\toprule
 & \textbf{Train} & \textbf{Dev} & \textbf{Test} \\ \midrule
\textbf{MetaQA 2} & $118\,980$ & $14\,872$ & $14\,872$ \\
\textbf{MetaQA 3} & $114\,196$ & $14\,274$ & $14\,274$ \\
\textbf{WebQSP} & $2\,848$ & $250$ & $1\,639$ \\
\bottomrule
\end{tabular}%
\end{table}

\begin{table}[!htb]
\caption{Size of the knowledge graphs used for \textit{MetaQA} and \textit{WebQuestionsSP}}
\centering
\label{tab:kg_stats}
\begin{tabular}{@{}lccc@{}}
\toprule
& \textbf{Triples} & \textbf{Entities} & \textbf{Relations}  \\
\midrule
\textbf{MetaQA} & $392\,906$ & $43\,230$ & $18$ \\
\textbf{WebQSP} & $23\,587\,078$ & $7\,448\,928$ & $575$  \\
\bottomrule
\end{tabular}%
\end{table}

\paragraph{Compositional generalization}
Our experiments on compositional generalization rely on the \emph{Compositional Freebase Questions} (\emph{CFQ}) dataset. It includes a total of $239\,357$ English question-answer pairs that are answerable using the public Freebase data \cite{freebase}. \textit{CFQ} is released under the CC-BY-4.0 license\footnote{\url{https://creativecommons.org/licenses/by/4.0/}} provides train-test splits designed to measure the compositional generalization ability of a machine-learning model.
Each question is composed of primitive elements (\textit{atoms}), which include entity mentions, predicates and question patterns. These atoms can be combined in different ways (\emph{compounds}) to instantiate the specific examples in the dataset.
The train-test splits are designed with the twofold goal of:
\begin{enumerate}
    \item \emph{minimizing atom divergence}: the atoms present in the test set are also included in the training set and their distribution in the test set is as similar as possible to their distribution in the test set;
    \item \emph{maximizing compound divergence}: the distribution of compounds in the test set is as different as possible from their distribution in the training set.
\end{enumerate}
The dataset provides three different splits (\textbf{MCD1}, \textbf{MCD2}, \textbf{MCD3}), with \emph{maximum compound divergence} (MCD) and low atom divergence. For each question, both a logical parsing and the expected answers are included. Hence, \textit{CFQ }can be used both for semantic parsing and end-to-end question answering.

\subsection{Baselines}
\label{app:baselines}

\paragraph{KBQA Baselines}
On \emph{WebQuestionsSP} and \emph{MetaQA}, we compared our approach against the following baselines.

\begin{itemize}
\item \textbf{KV-MemNN} is a key-value memory network \citep{Miller2016} that makes use of a memory of key-value pairs to store the triples from the KG. Keys are joint representation of the subject and relation of each triple, whereas the objects of the triples are used as the corresponding values.
\item \textbf{ReifKB} \citep{Cohen2020} uses a compact encoding for representing symbolic KGs, called a sparse-matrix reified KG, which can be distributed across multiple GPUs, allowing efficient symbolic reasoning.
\item \textbf{GRAFT-Net} \citep{Sun2018} is
a graph neural network designed to reason over question-specific subgraphs. The message-passing scheme is conditioned on the input question and takes inspiration from personalized page rank to perform a directed propagation of the messages starting from the entities mentioned in the question.
%
\item \textbf{PullNet} \citep{Sun2019} builds on top of GRAFT-Net and improves the quality of the question-specific subgraphs with an iterative process based on a learned classifier. This classifier selects which node should be expanded at each iteration and it is a further GNN with the same architecture as GRAFT-Net.

\item \textbf{EmbedKGQA} \citep{Saxena2020} uses KG embeddings for multi-hop question answering.
%
\item
\textbf{EmQL} \cite{Sun2020} relies on a query embedding method that combines a count-min sketch representation for entity sets with logical operations implemented via neural retrieval over embedded KG triples.
\end{itemize}

\paragraph{\emph{CFQ} Baselines}

For the experiment on compositional generalization, we compare to the best-performing baselines in \textit{CFQ}'s public leaderboard\footnote{\url{https://github.com/google-research/google-research/tree/master/cfq}}. These baselines are all designed for semantic parsing and are encoder-decoder architectures trained to output a formal query given a natural language question.
\citet{Keysers2020} evaluated the compositional generalization capabilities of three sequence-to-sequence models, namely one
based on LSTMs \cite{hochreiter1997long} equipped with an attention mechanism \cite{Bahdanau2015} (\textbf{LSTM + Attention}), a \textbf{Transformer} \cite{Vaswani2017} and a \textbf{Universal Transformer} \citep{Dehghani2019}.
\citet{Furrer2020} conducted a study that assessed the performance of three more models. The
\textbf{Evolved Transformer} \citep{So2019} is a variation of the Transformer discovered with an evolutionary neural architecture search seeded with the original model of \citet{Vaswani2017}.
The \textit{Text-to-Text Transfer Transformer} (T5) \cite{Raffel2020} is a model pre-trained to treat every task as a text-to-text problem. \citet{Furrer2020} fine-tuned all variants, including the largest one with 11 billion parameters (\textbf{T5-11B}).
Following the technique introduced by \citet{Guo2019}, \citet{Furrer2020} further implemented the variant
\textbf{T5-11B-mod}, which learns to predict an intermediate representation of the SPARQL query which is closer to the formulation of the questions in natural language.
Finally, \citet{Guo2020} introduced the \textit{Hierarchical Poset Decoding} (\textbf{HPD}), which enforces partial permutation invariance, thus taking into account semantics and capturing higher-level compositionality.

\subsection{Hyperparameters and Reproducibility}
\label{app:hyperparameters}
We train the relation-level model for 300 epochs on both datasets. We use a mini-batch size of 128 for \textit{MetaQA} and 32 for \textit{WebQuestionsSP}. We set the dimension of the embeddings to $d_{\textit{model}} = 768$, as we use 12 attention heads applied to tensors of size $64$.
We optimize the model using the AdamW optimizer \cite{Loshchilov2019}, with weight decay of $10^{-3}$. The initial learning rate is set to $10^{-4}$ for \textit{MetaQA} and $5 \cdot 10^{-6}$ for \textit{WebQuestionsSP}.
We apply dropout regularization, with probability $0.1$ on both the encoder and the decoder layers.
We use a beam width $\beta = 10$ for the knowledge seeking algorithm described in Appendix \ref{app:computational_complexity}.

BERT is fine-tuned for \textit{WebQuestionsSP}, whereas the weights are kept fixed for \textit{MetaQA}.
For experiments on \textit{WebQuestionsSP}, we found beneficial to pretrain our model in order to incorporate knowledge from Freebase into the layers of the decoder, as explained in Appendix \ref{app:training}.
Specifically, we sampled a total of 500k 1-hop or 2-hop paths and we trained the model to predict the sequence of relations connecting two nodes, given the embeddings of the source node and the target node of the path.

For the GCN-based edge-level model, we used the same implementation and hyperparameters of the version publicly available at: \url{https://github.com/OceanskySun/GraftNet}. All experiments are performed on a NVIDIA Tesla V100 GPU with 16 GB of memory.

\subsection{Discussion and qualitative examples}
In our experiments on KBQA, for each model, we selected the entity $v^\star \in \nodes$ with the highest likelihood to be a correct answer.
The answer to the question is considered correct if $v^\star \in \answers$.
For the unrefined \ours{} model, we report the expected performance selecting $v^\star$ uniformly at random from $\candidates$.

The high performance of the unrefined model on \emph{MetaQA} confirms our hypothesis that the relation-level model applied on the coalesced representation is sufficient for tasks such as multi-hop question answering.
On \textit{WebQuestionsSP}, the edge-level model is needed because relation projection is not sufficient to answer some of the questions in the dataset. Figure \ref{fig:webqsp_path_example} shows some examples of the relation sequences predicted by our model on the test set of \emph{WebQuestionsSP}. 
For examples \textit{(a)} and \textit{(b)}, the relation-level model is sufficient, as the set of candidates $\candidates$ is the same as the set of the actual answers $\answers$. However, examples \textit{(c)} and \textit{(d)} demonstrate the need for an edge-level model, as following a sequence of relations is not always sufficient to obtain the correct answer.
Note that the edge-level model is applied on a 1-hop neighborhood expansion of the graphs depicted in Figure \ref{fig:webqsp_path_example} and constrained to select an answer among the candidates $\candidates$.
Figure \ref{fig:webqsp_path_example} also shows that the answer paths for 2-hop questions in \emph{WebQuestionsSP} always contain compound value type (CVT) entities in the middle (depicted with cyan nodes in the image). These are special entity types that are used in \textit{Freebase} to describe $n$-ary relationships between entities.
EmQL uses different encodings for CVT nodes and the real entities, while \ours{} does not depend on the KG specifics.
\begin{figure*}[!t]
    \centering
    \includegraphics[width=\linewidth]{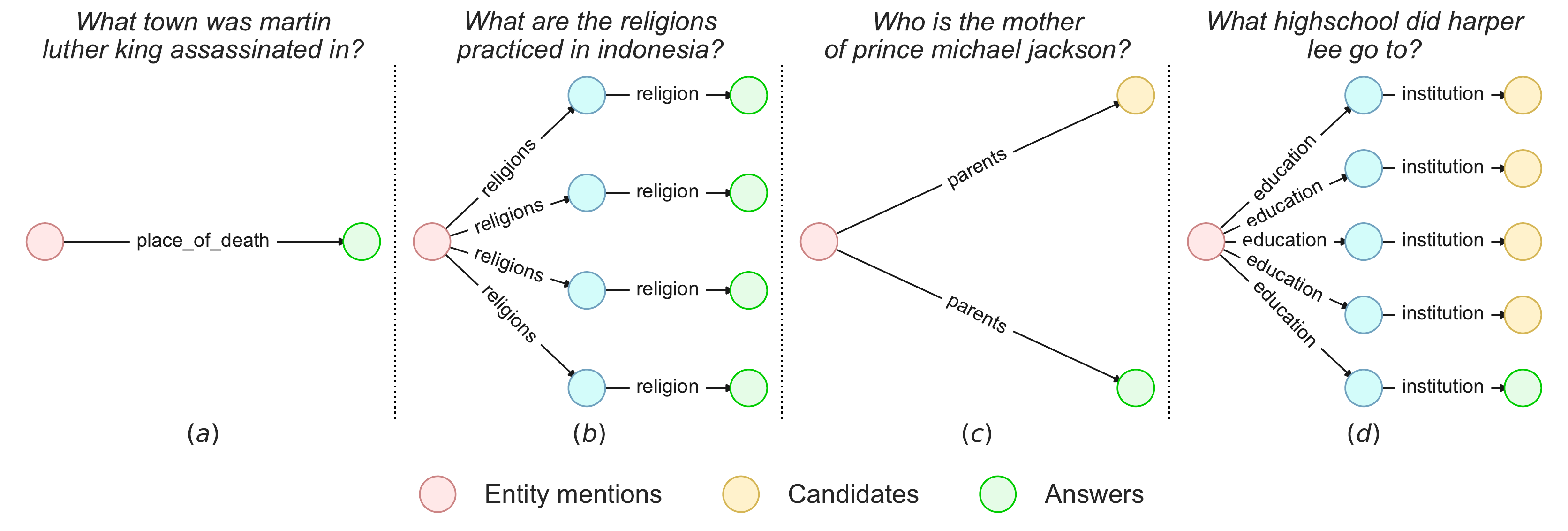}
    \caption{Example of relation sequences predicted by the unrefined relation-level model on the test set of \textit{WebQuestionsSP}}
    \label{fig:webqsp_path_example}
\end{figure*}

\subsection{Analysis of relational coalescing}

\begin{figure}[!b]
\centering
\begin{subfigure}{0.49\linewidth}
\centering
\includegraphics[width=0.8\linewidth]{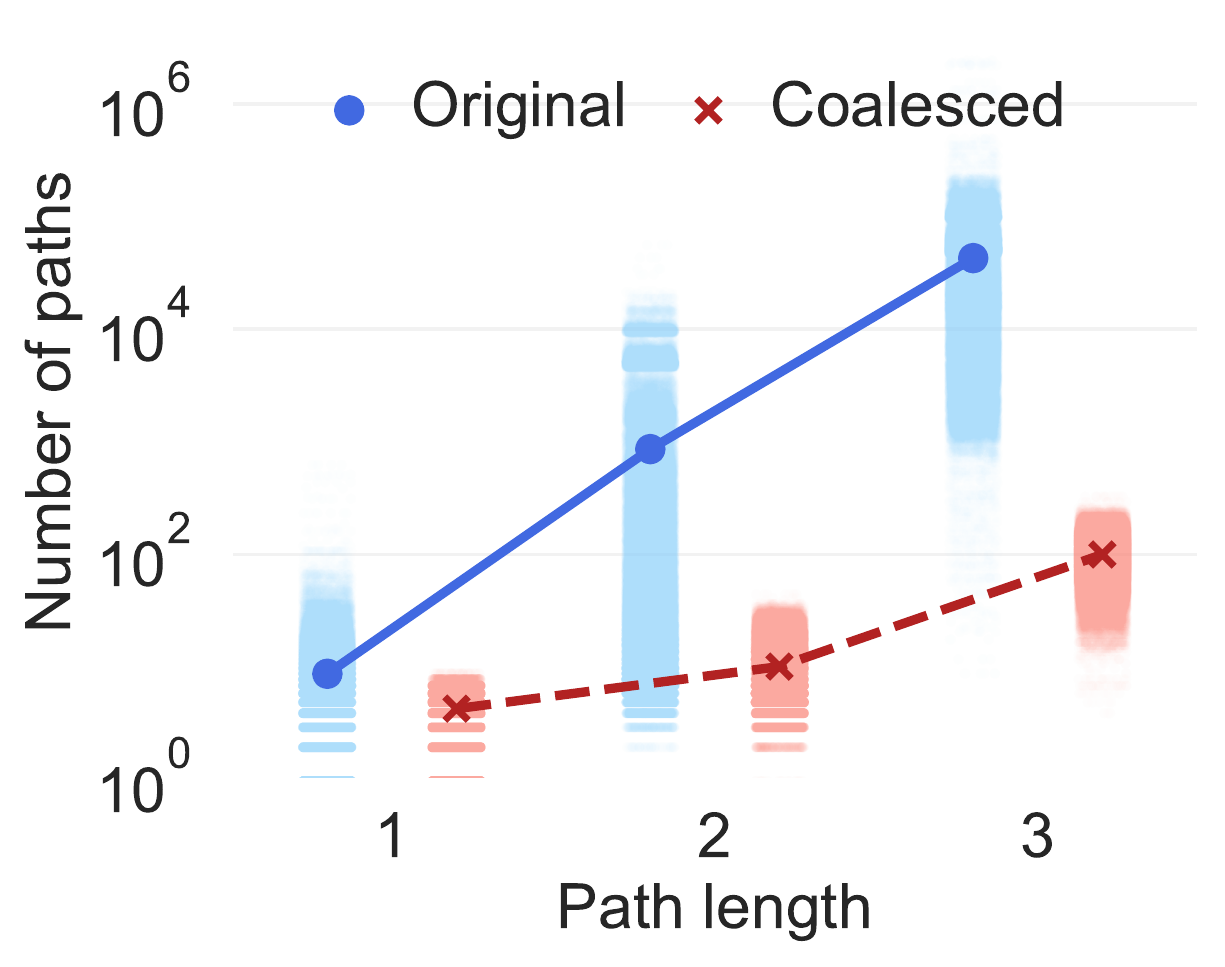}
\end{subfigure}
\begin{subfigure}{0.49\linewidth}
\centering
\includegraphics[width=0.8\linewidth]{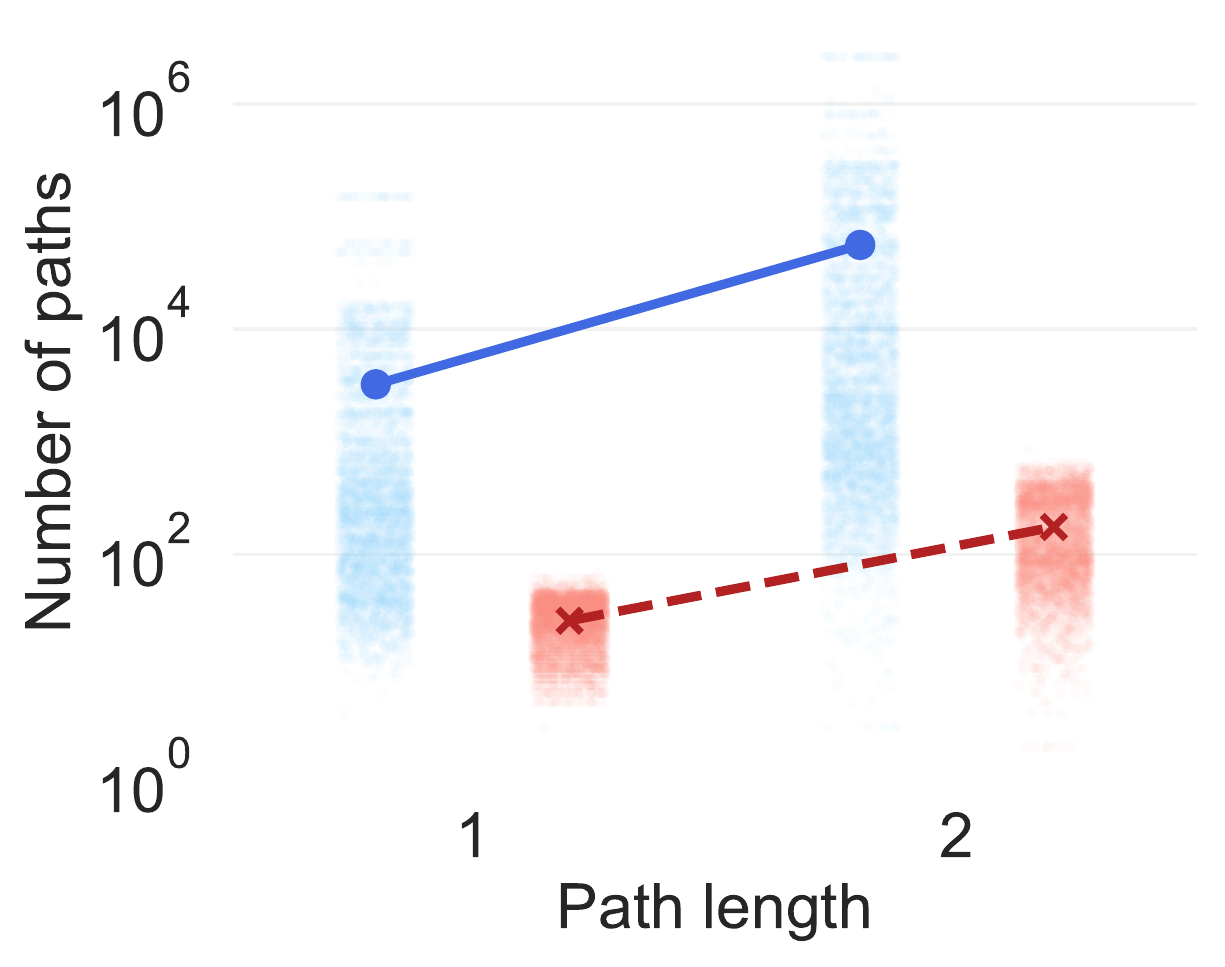}
\end{subfigure}
\caption{Number of paths by path length in the original and coalesced graphs for \textit{MetaQA} (right) and \textit{WebQuestionsSP} (left)}
\label{fig:coalescing_stats}
\end{figure}

\label{app:scalability_coalescing}
In order to answer a question that requires multi-hop reasoning over a KG correctly, one should ideally either consider the full KG or a complete subgraph consisting of all possible multi-hop neighbors of the entities mentioned in the question. However, such subgraphs can be very large, as shown in Table \ref{tab:subgraph_sizes}.
The subgraphs we analyzed include 3-hop neighbors for \textbf{MetaQA 3} and both 1 and 2-hop neighbors of the entities mentioned in the question for \textbf{WebQSP}.
The average number of nodes for MetaQA 3 exceeds 10k and for the larger Freebase KG there are question subgraphs with millions of nodes.
This makes impractical to perform KBQA on complete subgraphs with models that scale with the number of edges or nodes in the graph.

\begin{table}[!htb]
\centering
\caption{Size of the subgraphs including all neighbors of the entities mentioned in a question}
\label{tab:subgraph_sizes}
\begin{tabular}{@{}lcc@{}}
\toprule
 & \textbf{MetaQA 3} & \textbf{WebQSP} \\ \midrule
\textbf{Mean nodes} &  8.6k & 36k \\
\textbf{Max nodes} & 30k &  1.6M \\
\midrule
\textbf{Mean facts} & 49k & 211k \\
\textbf{Max facts}  & 230k & 9.5M \\ \bottomrule
\end{tabular}%
\end{table}

As a further analysis, we investigate the computational advantage of relational coalescing by computing the number of paths originating from the entities mentioned in the questions both in the original KG and in the coalesced representation.
Figure \ref{fig:coalescing_stats} presents the results for both \emph{MetaQA} and \emph{WebQuestionsSP}.
The experiment shows that relation coalescing allows reducing the number of paths by up to 2 orders of magnitude on both datasets. This directly impacts both the memory requirements and the efficiency of our approach.
For \emph{MetaQA} we analyzed paths of length up to $3$, whereas for \emph{WebQuestionsSP} we consider paths of length $1$ or $2$, as the dataset does not include $3$-hop questions.